\documentclass[11pt,reqno]{amsart}
\usepackage[top=1.2in, bottom=1.2in, left=1.2in, right=1.2in]{geometry}
\usepackage{amsfonts, amssymb, amsmath, mathtools, dsfont}
\usepackage[linktoc=page, colorlinks, linkcolor=blue, citecolor=blue]{hyperref}
\usepackage{enumerate, commath}

\usepackage{graphicx,xcolor, abstract}
\usepackage[font=small]{caption} 

\usepackage{algorithm,algpseudocode}

\numberwithin{equation}{section}

\mathtoolsset{showonlyrefs}

\DeclareMathOperator{\E}{\mathbb{E}}

\DeclareMathOperator{\tr}{tr}

\DeclareMathOperator{\Lap}{Lap}

\DeclareMathOperator{\den}{\mathrm{den}}

\DeclareMathOperator*{\argmin}{arg\,min}

\renewcommand{\Pr}[2][]{\mathbb{P}_{#1} \left\{\, #2 \rule{0mm}{3mm}\right\}}

\def \R {\mathbb{R}}

\def \Z {\mathbb{Z}}

\def \e {\varepsilon}
\def \d {\delta}
\def \l {\lambda}
\def \s {\sigma}

\def \top {\mathsf{T}}

\newcommand{\add}[1]{{#1}}

\newtheorem{theorem}{Theorem}[section]
\newtheorem{proposition}[theorem]{Proposition}

\newtheorem{lemma}[theorem]{Lemma}

\newtheorem{definition}[theorem]{Definition}

\newtheorem{remark}[theorem]{Remark}

\newcommand{\cA}{\mathcal{M}}

\newcommand{\cR}{\mathcal{R}}

\newcommand{\dd}{\mathrm{d}}

\renewcommand{\>}{\right>}

\newcommand{\1}{\mathbf{1}}
\newcommand{\bA}{\mathbf{A}}
\newcommand{\bM}{\mathbf{M}}
\newcommand{\bV}{\mathbf{V}}
\newcommand{\bX}{\mathbf{X}}
\newcommand{\bY}{\mathbf{Y}}
\newcommand{\bZ}{\mathbf{Z}}

\title{Differentially Private Low-dimensional Synthetic Data from High-dimensional Datasets}

\renewcommand{\paragraph}[1]{\subsection*{#1}}

\author{Yiyun He}
\address{Department of Mathematics, University of California Irvine}
\email{yiyunh4@uci.edu}

\author{Thomas Strohmer}
\address{Department of Mathematics  and Center of Data Science
and Artificial Intelligence Research, University of California Davis}
\email{strohmer@math.ucdavis.edu}

\author{Roman Vershynin}
\address{Department of Mathematics, University of California Irvine}
\email{rvershyn@uci.edu}

\author{Yizhe Zhu}
\address{Department of Mathematics, University of California Irvine}
\email{yizhe.zhu@uci.edu}

\usepackage{times}

\begin{document}

\maketitle

\begin{abstract}
    {Differentially private synthetic data provide a powerful mechanism to enable data analysis while protecting sensitive information about individuals. However, when the data lie in a high-dimensional space, the accuracy of the synthetic data suffers from the curse of dimensionality. In this paper, we propose a differentially private algorithm to generate low-dimensional synthetic data efficiently from a high-dimensional dataset with a utility guarantee with respect to the Wasserstein distance. A key step of our algorithm is a private principal component analysis (PCA) procedure with a near-optimal accuracy bound that circumvents the curse of dimensionality. Unlike the standard perturbation analysis, our analysis of private PCA works without assuming the spectral gap for the covariance matrix.}
\end{abstract}

\section{Introduction}\label{sec:intro}
As data sharing is increasingly locking horns with data privacy concerns, privacy-preserving data analysis is becoming a challenging task with far-reaching impact.
Differential privacy (DP) has emerged as the gold standard for implementing privacy in various applications \cite{dwork2014algorithmic}.  For instance, DP has been adopted by several technology companies~\cite{dwork2019differential} and has also been used in connection with the release of Census 2020 data~\cite{abowd20222020}.  The motivation behind the concept of differential privacy is the desire to
protect an individual's data while publishing aggregate information about the database, as formalized in the following definition:
\begin{definition}[Differential Privacy~\cite{dwork2014algorithmic}]  \label{def: DP}
  A randomized algorithm ${\mathcal M}$ is $\e$-differentially private
 if for any \add{pair of} datasets $D$ and $D'$ \add{that differ on one data (i.e. $D=D_0\cup \{X\}$ and $D'=D_0\cup\{X'\}$ for some dataset $D_0$), } 
 and any measurable subset $S \subseteq \textnormal{range}({\mathcal M})$, we have
 $$
 \Pr{{\mathcal M}(D) \in S} \le e^{\e} \, \Pr{{\mathcal M}(D') \in S},
 $$
where the probability is with respect to the randomness of  ${\mathcal M}$.
\end{definition}

However, utility guarantees for DP are usually provided only for a fixed, predefined set of queries.  Hence, it has been frequently recommended that differential privacy may be combined with synthetic data to achieve more flexibility in private data sharing~\cite{wasserman2010statistical,hardt2012simple,bellovin2019privacy}. Synthetic datasets are generated from existing datasets and maintain the statistical properties of the original dataset. Hence, the datasets can be shared freely
among investigators in academia or industry, without security and privacy concerns.

Yet, computationally efficient construction of accurate differentially private synthetic data is challenging.
Most research on private synthetic data has been concerned with counting queries, range queries, or $k$-dimensional marginals, see, e.g.,~\cite{hardt2012simple,ullman2011pcps,blum2013learning,vietriprivate,dwork2015efficient,thaler2012faster,boedihardjo2022private2}. Notable exceptions are~\cite{wang2016differentially,boedihardjo2022private,donhauser2023sample}.
Specifically,~\cite{boedihardjo2022private} provides utility guarantees with respect to the 1-Wasserstein distance. Invoking the Kantorovich-Rubinstein duality theorem, the  1-Wasserstein distance accuracy bound ensures that all Lipschitz statistics are preserved uniformly. Given that numerous machine learning algorithms are Lipschitz \cite{von2004distance,kovalev2022lipschitz,bubeck2021universal,meunier2022dynamical}, this provides data analysts with a vastly increased toolbox of machine learning methods for which one can expect similar outcomes for the original and synthetic data.  

For instance, for the special case of datasets living on the $d$-dimensional Boolean hypercube \add{$\{0,1\}^d$} equipped with the Hamming distance, the results in~\cite{boedihardjo2022private} show that there exists an $\e$-DP algorithm with an expected utility loss that scales like
\begin{equation}        \label{eq:utility1}
\left(\log(\e n)^\frac{3}{2}/(\e n) \right)^{1/d},
\end{equation}
where $n$ is the size of the dataset. While~\cite{he2023algorithmically} succeeded in removing the logarithmic factor in~\eqref{eq:utility1}, it can be shown that the rate in~\eqref{eq:utility1} is otherwise tight. Consequently, the utility guarantees in~\cite{boedihardjo2022private,he2023algorithmically} are only useful when $d$, the dimension of the data, is small (or if $n$ is exponentially larger than $d$). In other words, we are facing the curse of dimensionality.
The curse of dimensionality extends beyond challenges associated with Wasserstein distance utility guarantees.  Even with a weaker accuracy requirement, the hardness result from Uhlman and Vadhan \cite{ullman2011pcps} shows that $n=\mathrm{poly}(d)$ is necessary for generating DP-synthetic data in polynomial time while maintaining approximate covariance.

In~\cite{donhauser2023sample}, the authors succeeded in constructing DP synthetic data with utility bounds where $d$ in \eqref{eq:utility1} is replaced by $(d'+1)$, assuming that the dataset lies in a certain $d'$-dimensional subspace. \add{Their notion of dimension is similar to the Minkowski dimension, and their method is applicable beyond the linear subspace setting}. However, the optimization step in their algorithm exhibits exponential time complexity in $d$, \add{see \cite[Section D]{donhauser2023sample}}. 

This paper presents a computationally efficient algorithm that does not rely on any assumptions about the true data. We demonstrate that our approach enhances the utility bound from $d$ to $d'$ in \eqref{eq:utility1} when the dataset is in a $d'$-dimensional affine subspace. Specifically, we derive a DP algorithm to generate low-dimensional synthetic data from a high-dimensional dataset with a utility guarantee with respect to the 1-Wasserstein distance that captures the intrinsic dimension of the data.

Our approach revolves around a private principal component analysis (PCA) procedure with a near-optimal accuracy bound that circumvents the curse of dimensionality. Different from classical perturbation analysis \cite{chaudhuri2013near,dwork2014analyze} that utilizes the Davis-Kahan theorem \cite{davis1970rotation} in the literature, our accuracy analysis of private PCA  works without assuming the spectral gap for the covariance matrix. 

\paragraph{Notation}
In this paper, we work with data in the Euclidean space $\R^d$.  For convenience, the data matrix $\bX=[X_1,\dots, X_n]\in \R^{d\times n}$ also indicates the dataset $(X_1,\dots,X_n)$.
 We use $\bA$ to denote a matrix and $v, X$ as vectors. $\|\cdot \|_F$ denotes the Frobenius norm and $\| \cdot \|$ is the operator norm of a matrix.  Two sequences $a_n,b_n$ satisfies $a_n\lesssim b_n$ if $a_n\leq Cb_n$ for an absolute constant $C>0$.

\paragraph{Organization of the paper} The rest of the paper is arranged as follows. In the remainder of  Section~\ref{sec:intro},  we present our algorithm with the \add{main theorem} for privacy and accuracy guarantees in Section \ref{sec:main_results}, followed by a discussion. A comparison to the state of the art is given in Section~\ref{sec:comparison}.  Definitions and lemmas used in the paper are provided in Section~\ref{sec: preliminaries}.

Next, we consider the Algorithm~\ref{alg: affine} step by step. Section~\ref{sec: projection} discusses private PCA and noisy projection. In Section~\ref{sec: subroutines}, we modify synthetic data algorithms from \cite{he2023algorithmically} to the specific cases on the lower dimensional spaces. The precise privacy and accuracy guarantee of Algorithm~\ref{alg: affine} is summarized in Section~\ref{sec: main theorem}. 
\add{We discuss an adaptive and private choice of $d'$ in Section~\ref{sec: d'choice}}.
Finally, since the case  $d'=1$ is not covered in Theorem \ref{thm:main}, we discuss additional results under stronger assumptions in Section~\ref{sec: d'=1}.


\subsection{Main results}\label{sec:main_results}
 In this paper, we use Definition~\ref{def: DP} on data matrix $\bX \in \R^{d\times n}$. We say two data matrices $\bX, \bX'$ are  \textit{neighboring datasets} if $\bX$ and $\bX'$ differ on only one column.
We follow the setting and notation in \cite{he2023algorithmically} as follows. let $(\Omega,\rho)$ be a metric space. Consider a dataset $\bX=[X_1,\dots ,X_n]\in \Omega^n$. We aim to construct a computationally efficient differentially private randomized algorithm that outputs synthetic data $\bY=[Y_1,\dots ,Y_n]\in \Omega^m$  such that the two empirical measures
\[  \mu_{\bX}=\frac{1}{n} \sum_{i=1}^n \delta_{X_i} \quad \text{and} \quad \mu_{\bY}=\frac{1}{m} \sum_{i=1}^m \delta_{Y_i}\]
 are close to each other. Here $\delta_{X_i}$ denotes the Dirac measure centered on $X_i$.
 
 We measure the utility of the output by $\E W_1(\mu_{\bX},\mu_{\bY})$, where the expectation is taken over the randomness of the algorithm.  
 We assume that each vector in the original dataset $\bX$ is inside  $[0,1]^d$; our goal is to generate a differentially private synthetic dataset $\bY$ in $[0,1]^d$, where each vector is close to a linear subspace of dimension $d'$, and the empirical measure of $\bY$ is close to $\bX$ under the 1-Wasserstein distance.  We introduce Algorithm~\ref{alg: affine} as a computationally efficient algorithm for this task.
 It can be summarized in the following four steps:
\begin{enumerate}
    \item Construct a private covariance matrix $\widehat{\bM}$. The private covariance is constructed by adding a Laplacian random matrix to a centered covariance matrix $\bM$ defined as  
    \begin{equation}
    \label{eq: sample_cov}
    \bM = \frac{1}{n-1}\sum_{i=1}^n (X_i - \overline{X})(X_i-\overline X)^\top, \quad \text{where} \quad \overline{X}=\frac{1}{n} \sum_{i=1}^n X_i. 
    \end{equation}
    This step is presented in Algorithm~\ref{alg: covariance}.
    \item Find a $d'$-dimensional subspace $\widehat{\bV}_{d'}$ by taking the top $d'$ eigenvectors of $\widehat{\bM}$. Then, project the data onto a linear subspace.   The new data obtained in this way are inside a $d'$-dimensional ball. This step is summarized in Algorithm~\ref{alg: projection}.
    \item Generate a private measure in the $d'$ dimensional ball centered at the origin by adapting methods in \cite{he2023algorithmically}, where synthetic data generation algorithms were analyzed for data in the hypercube. This is summarized in Algorithms~\ref{alg: pmm} and ~\ref{alg: psmm}.
    \item Add a private mean vector to shift the dataset back to a private affine subspace. Given the transformations in earlier steps, some synthetic data points might lie outside the hypercube. We then metrically project them back to the domain of the hypercube. Finally, we output the resulting dataset $\bY$.  This is summarized in the last two parts of Algorithm~\ref{alg: affine}.
\end{enumerate}



\add{Our main theorem} states the privacy and accuracy guarantees of Algorithm \ref{alg: affine}. 

\begin{theorem}\label{thm:main}
    Let $\Omega=[0,1]^d$ equipped with $\ell^{\infty}$ metric and $\bX=[X_1,\dots, X_n]\in \Omega^n$ be a dataset. 
    For any $2\leq d'\leq d$, Algorithm~\ref{alg: affine} outputs an $\e$-differentially private synthetic dataset $\bY=[Y_1,\dots, Y_m]\in \Omega^{m}$ for some $m\geq 1$  in polynomial time such that 
       \begin{equation}
    \label{eq:main_accuracy}
        \E W_1(\mu_{\bX},\mu_{\bY}) \lesssim \sqrt{\sum_{i> d'}\sigma_i(\bM)} +\sqrt{\frac{d' d^{2.5}}{\e n}}+ \sqrt{\frac{d}{d'}}(\e n)^{-1/d'},
    \end{equation}
    where  $\s_i(\bM)$ is the $i$-th largest eigenvalue value of $\bM$ in \eqref{eq: sample_cov}.
\end{theorem}

Note that $m$, the size of the synthetic dataset $\bY$, is not necessarily equal to $n$ since the low-dimensional synthetic data subroutine in Algorithm \ref{alg: affine} creates noisy counts. See Section \ref{sec: subroutines} for more details.

\begin{algorithm}			
    \caption{Low-dimensional Synthetic Data} 
    \label{alg: affine}
    \begin{algorithmic}
    \State {\bf Input:}  True data matrix $\bX=[X_1,\dots,X_n]$, $X_i\in [0,1]^d$, privacy parameter $\e$. 
    
    \State {\bf (Private covariance matrix)} Apply Algorithm~\ref{alg: covariance} to $\bX$ with privacy parameter  $\e/3$ to obtain a private covariance matrix $\widehat \bM$.
     
    \State{\bf (Private linear projection)} 
    \add{Let $\overline{X}_{\textrm{priv}}$ denote the private mean of the true dataset.} Choose a target dimension $d'$. Apply Algorithm~\ref{alg: projection} with privacy parameter  $\e/ 3$ to \add{shift and} project $\bX$ onto a private $d'$-dimensional linear subspace.  
    
    \State{\bf (Low-dimensional synthetic data)} Use subroutine in Section~\ref{sec: subroutines} to generate  $\e/3$-DP synthetic data $\bX'$ of size $m$ depending on $d'=2$ or $d'\geq 3$.
    
    \State{\bf (Adding the private mean vector)}
      Shift the data back by $X_i''= \add{X'_i} + \overline{X}_{\textrm{priv}}$.
    
    \State{\bf (Metric projection)}   Define  $f:\R\to [0,1]$ such that
    \[f(x)=\left\{\begin{aligned}
        &0 \quad\textrm{ if } x<0;\\
        &x \quad\textrm{ if } x\in [0,1];\\
        &1 \quad\textrm{ if } x> 1.
    \end{aligned}\right.\]
    Then, for $v\in \R^d$, we define $f(v)$ to be the result of applying $f$ to each coordinate of $v$.
    
    \State {\bf Output:} Synthetic data $\bY = [f(X_1''),\dots, f(X_m'')]$.
    \end{algorithmic}
\end{algorithm}


\paragraph{Optimality} 

There are three terms on the right-hand side of \eqref{eq: three-term}.    The first term is the error from the rank-$d'$ approximation of the covariance matrix $\bM$.  The second term is the accuracy loss for private PCA after the perturbation from a random Laplacian matrix.  The optimality of this error term remains an open question.   The third term is the accuracy loss when generating synthetic data in a $d'$-dimensional subspace. Notably, the factor $\sqrt{d/d'}$ is  optimal. 
This can be seen by the fact that  a $d'$-dimensional section of the cube can be $\sqrt{d/d'}$ 
times larger than the low-dimensional cube $[0,1]^{d'}$ (e.g., if it is positioned diagonally).  Complementarily,  \cite{boedihardjo2022private} showed  the optimality of the factor $(\e n)^{-1/d'}$ 
 for generating $d'$-dimensional synthetic data in $[0,1]^{d'}$. Therefore,  the third term in \eqref{eq: three-term} is necessary and optimal.

\paragraph{Improved accuracy}
    When the original dataset $\bX$ lies in an affine $d'$-dimensional subspace, it implies $\sigma_i(\bM)=0$ for $i>d'$ and 
  \add{ $
    \E W_1(\mu_{\bX},\mu_{\bY}) \lesssim \sqrt{\frac{d' d^{2.5}}{\e n}}+ \sqrt{\frac{d}{d'}}(\e n)^{-1/d'}$.}
      This is an improvement from the accuracy rate $O((\e n)^{-1/d})$ for unstructured data in $[0,1]^d$ in  \cite{boedihardjo2022private,he2023algorithmically} \add{when $d\leq n^{\alpha_n}$ and $d'\leq \min\{\frac{d}{2},\frac{1}{\alpha_n}\}$ for $0<\alpha_n\leq \frac{2}{7}$. For example, we can take $\alpha_n$ to be a constant in $(0,\frac{2}{7}]$ or $\alpha_n=\frac{1}{\log\log n}$. This improved rate overcomes the curse of
high dimensionality.}

\paragraph{Adaptive and private  choices of $d'$} 
    \add{The target dimension $d'$ is a hyperparameter in Algorithm~\ref{alg: affine}.} One can choose the value of $d'$ adaptively and privately based on singular values of  the private covariance matrix $\widehat\bM$ in Algorithm~\ref{alg: covariance} such that
    \add{
    \[d'\coloneqq \argmin_{2\leq k\leq d} \Bigg(\sqrt{\sum_{i> d'}\sigma_i(\widehat \bM)} + \sqrt{\frac{d}{d'}}(\e n)^{-1/d'} \Bigg).\]
    Discussion on such choice of $d'$ is referred to Section~\ref{sec: d'choice}.
    }

\paragraph{Low-dimensional  representation of $\bX$.}
    The synthetic dataset $\bY$ is close to a  $d'$-dimensional subspace under the 1-Wasserstein distance, as shown in Proposition \ref{prop:Y_is_close}.

\paragraph{Running time} The \textit{private linear projection} step in Algorithm \ref{alg: affine} has a running time $O(d^2n)$ using the truncated SVD \cite{li2019tutorial}.  The \textit{low-dimensional synthetic data} subroutine has a running time polynomial in $n$  for $d'\geq 3$ and linear in $n$ when $d'=2$ \cite{he2023algorithmically}.
Therefore, the overall running time for Algorithm~\ref{alg: affine} is linear in $n$, polynomial in $d$ when $d'=2$ and is  $\mathrm{poly}(n,d)$ when $d'\geq 3$.
Although sub-optimal in the dependence on $d'$ for accuracy bounds, one can also run Algorithm \ref{alg: affine}  in linear time by choosing PMM (Algorithm \ref{alg: pmm}) in the subroutine for all $d'\geq 2$.

\subsection{Comparison to previous results}\label{sec:comparison}

\paragraph{Private synthetic data}
 Most existing work considered generating DP-synthetic datasets while minimizing the utility loss for specific queries, including counting queries \cite{blum2013learning,hardt2012simple,dwork2009complexity}, $k$-way marginal queries \cite{ullman2011pcps,dwork2015efficient},  histogram release \cite{abowd2019census}. For a finite collection of predefined linear queries $Q$, \cite{hardt2012simple} provided an algorithm with running time linear in $|Q|$ and utility loss grows logarithmically in $|Q|$.
  The sample complexity can be reduced if the queries are sparse  \cite{dwork2015efficient,blum2013learning,donhauser2023sample}. 
Beyond finite collections of queries, \cite{wang2016differentially} considered utility bound for differentiable queries,
 and recent works \cite{boedihardjo2022private,he2023algorithmically} studied Lipschitz queries with utility bound in Wasserstein distance. \cite{donhauser2023sample} considered sparse Lipschitz queries with an improved accuracy rate. \cite{balog2018differentially,harder2021dp,kreacic2023differentially,yang2023differentially} measure the utility of DP synthetic data by the maximum mean discrepancy (MMD) between empirical distributions of the original and synthetic datasets. This metric is different from our chosen utility bound in Wasserstein distance. Crucially, MMD does not provide any guarantees for Lipschitz downstream tasks.

Our work provides an improved accuracy rate for low-dimensional synthetic data generation. Compared to \cite{donhauser2023sample}, our algorithm is computationally efficient and has a better accuracy rate. Besides \cite{donhauser2023sample}, we are unaware of any work on low-dimensional synthetic data generation from high-dimensional datasets.  
While methods from \cite{boedihardjo2022private,he2023algorithmically} can be directly applied if the low-dimensional subspace is known,  the subspace would be non-private and could reveal sensitive information about the original data. The crux of our paper is that we do not assume the low-dimensional subspace is known, and our DP synthetic data algorithm protects its privacy. \add{This setting is closely related to the problem of privately learning the subspace of the dataset considered in \cite{dwork2014analyze,singhal2021privately,tsfadia2024differentially}.}

\paragraph{Private PCA}
Private PCA is a commonly used technique for differentially private dimension reduction of the original dataset. This is achieved by introducing noise to the covariance matrix  \cite{mangoubi2022re, chaudhuri2013near, imtiaz2016symmetric, dwork2014analyze,jiang2016wishart, jiang2013differential, zhou2009differential}. Instead of independent noise, the method of exponential mechanism is also extensively explored \cite{kapralov2013differentially,chaudhuri2013near,jiang2016wishart}.  Another approach, known as streaming PCA \cite{oja1982simplified, jain2016streaming},    can also be performed privately \cite{hardt2014noisy, liu2022dp}.

The private PCA typically yields a private $d'$-dimensional
subspace $\widehat{\bV}_{d'}$ that approximates the top $d'$-dimensional subspace $\bV_{d'}$ produced by the standard PCA. The accuracy of private PCA is usually measured by the distance between $\widehat{\bV}_{d'}$ and $\bV_{d'}$ \cite{dwork2014analyze, hardt2013beyond, mangoubi2022re, liu2022dp,singhal2021privately}. To prove a utility guarantee, a common tool is the Davis-Kahan Theorem \cite{bhatia2013matrix, yu2015useful}, which assumes that the covariance matrix has a spectral gap \cite{chaudhuri2013near, dwork2014analyze, hardt2014noisy, jiang2016wishart, liu2022dp}. Alternatively, using the projection error to evaluate accuracy is independent of the spectral gap  \cite{kapralov2013differentially, liu2022differential,arora2018differentially}. In our implementation of private PCA, we don't treat  $\widehat{\bV}_{d'}$ as our terminal output. Instead, we project $\bX$ onto  $\widehat{\bV}_{d'}$. Our approach directly bound the Wasserstein distance between the projected dataset and $\bX$. This method circumvents the subspace perturbation analysis,   resulting in an accuracy bound independent of the spectral gap, as outlined in  Lemma \ref{lem:SVD}.   \cite{singhal2021privately} considered a related task that takes a true dataset close to a low-dimensional linear subspace and outputs a private linear subspace. To the best of our knowledge, none of the previous work on private PCA  considered low-dimensional DP synthetic data generation.


\paragraph{Centered covariance matrix}
A common choice of the covariance matrix for PCA  is $\frac{1}{n} \bX\bX^\top$ \cite{chaudhuri2011differentially,dwork2014analyze,singhal2021privately}, which is different from the centered one defined in \eqref{eq: sample_cov}. 
The rank of $\bX$ is the dimension of the linear subspace that the data lie in rather than that of the affine subspace. If $\bX$ lies in a $d'$-dimensional affine space (not necessarily passing through the origin), centering the data shifts the affine hyperplane spanned $\bX$ to pass through the origin. Consequently, the centered covariance matrix will have rank $d'$, whereas the rank of $\bX$ is $d'+1$. By reducing the dimension of the linear subspace by $1$, the centering step  enhances the accuracy rate from $(\e n)^{-1/(d'+1)}$ to $(\e n)^{-1/d'}$ . Yet, this process introduces the challenge of protecting the privacy of mean vectors, as detailed in the third step in Algorithm~\ref{alg: affine} and Algorithm~\ref{alg: projection}.

\paragraph{Private covariance estimation}
Private covariance estimation \cite{dong2022differentially,mangoubi2022re} is closely linked to the private covariance matrix and the private linear projection components of our Algorithm~\ref{alg: affine}. Instead of adding i.i.d. noise, \cite{kapralov2013differentially,amin2019differentially} improved the dependence on $d$ in the estimation error by sampling top eigenvectors with the exponential mechanism. However, it requires $d'$ as an input parameter (in our approach, it can be chosen privately) and a lower bound on $\sigma_{d'}(\bM)$. The dependence on $d$ is a critical aspect in private mean estimation \cite{kamath2019privately, liu2021robust}, and it is an open question to determine the optimal dependence on $d$ for low-dimensional synthetic data generation.
 




\section{Preliminaries}
\label{sec: preliminaries}

\subsection{Differential Privacy}
We use the following definition of $\e$-differential privacy from \cite{dwork2014algorithmic}. Note that in particular, if the algorithm is $\cA: \Omega^n \to \Omega^m$, then its output is also a dataset of size $m$, which is generated by $\cA$ from the input real dataset. We say the synthetic dataset provides $\e$-differential privacy if the synthetic data algorithm $\cA$ is differentially private.
\begin{definition}[Differential privacy]
    A randomized algorithm $\cA:\Omega^n\to\cR$ provides \textit{$\e$-differential privacy} if for any input data $D,D'$ that differs on only one element (or $D$ and $D'$ are adjacent datasets) and for any measurable set $S\subseteq\mathrm{range}(\cA)$, there is 
    \[{\mathbb{P}\{\cA(D)\in S\}}\leq e^\e\cdot {\mathbb{P}\{\cA(D')\in S\}}.\]
Here the probability is taken from the probability space of the randomness of $\cA$.
\end{definition}

 
For multiple differentially private algorithms, differential privacy has a useful property that their sequential composition is also differentially private \cite[Theorem 3.16]{dwork2014algorithmic}. 

\begin{lemma}[Theorem 3.16 in \cite{dwork2014algorithmic}]
Suppose $\cA_i$ is $\e_i$-differentially private for $i=1,\dots,m$, then the sequential composition $x\mapsto(\cA_1(x),\dots,\cA_m(x))$ is $\sum_{i=1}^m\e_i$-differentially private.
\end{lemma}

Moreover, the following result about  \textit{adaptive composition} indicates that  algorithms in a sequential composition can use the outputs in the previous steps:
\begin{lemma}[Theorem 1 in \cite{dwork2006our}]
\label{lem: composition}
Suppose a randomized algorithm $\cA_1(x): \Omega^n\to \cR_1$ is $\e_1$-differentially private, and $\cA_2(x,y): \Omega^n\times \cR_1\to \cR_2$ is $\e_2$-differentially private with respect to the first component for any fixed $y$. Then the sequential composition 
\[x\mapsto(\cA_1(x),\cA_2(x,\cA_1(x)))\] 
is $(\e_1+\e_2)$-differentially private.
\end{lemma}

Since our method involves private counts of data points, we will use integer Laplacian noise to ensure they are integers.
\begin{definition}[Integer Laplacian distribution, \cite{inusah2006discrete}]
    An \textit{integer (or discrete) Laplacian distribution} with parameter $\s$ is a discrete distribution on $\mathbb Z$ with probability density function 
    \[f(z) = \frac{1-p_\s}{1+p_\s} \exp \left( -\abs{z}/\s \right),
    \quad z \in \Z,\]
    where $p_\s = \exp(-1/\s)$.
    A random variable $Z \sim \Lap_\Z(\s)$ is mean-zero and sub-exponential with variance $\mathrm{Var(Z)\leq 2\s^2}$.
\end{definition}


\subsection{Wasserstein distance}
The formal definition of $p$-Wasserstein distance is given as follows:
\begin{definition}[$p$-Wasserstein distance]
    Consider a metric space $(\Omega,\rho)$. The \textit{$p$-Wasserstein distance} (see e.g., \cite{villani2009optimal} for more details) between two probability measures $\mu,\nu$ is defined as 
\[W_p(\mu,\nu):=\left(\inf_{\gamma\in\Gamma(\mu,\nu)} \int_{\Omega\times\Omega}\rho(x,y)^p\dd\gamma(x,y)\right)^{1/p},\]
where $\Gamma(\mu,\nu)$ is the set of all couplings  of $\mu$ and $\nu$. 
\end{definition}
In particular, when $p=1$, the $W_1$ distance is also known as the earth mover's distance because it is equivalent to the optimal transportation problem if the probability measures are discrete. Furthermore, $W_1$ has the following Kantorovich-Rubinstein duality (see, e.g., \cite{villani2009optimal}), which gives an equivalent representation with the Lipschitz functions:
\begin{equation}
\label{eq: KR duality}
    W_1(\mu,\nu) = \sup_{{\mathrm{Lip}(f)}\leq 1} \left(\int f\dd\mu-\int f\dd \nu\right).
\end{equation}
Here the supremum is taken over the set of all $1$-Lipschitz functions on $\Omega$.


\section{Private linear projection}
\label{sec: projection}

\subsection{Private centered covariance matrix}
\begin{algorithm}
    \caption{Private Covariance Matrix} \label{alg: covariance}
    \begin{algorithmic}
    \State {\bf Input:}  Matrix $\bX=[X_1,\dots, X_n]$, privacy parameter $\e$, and variance parameter $\s = \frac{3d^2}{\e n}$.
    
        \State{\bf (Computing the covariance matrix)} Compute the mean $\overline X = \frac{1}{n}\sum_{i=1}^n X_i$ and the centered covariance matrix $\bM$.
        
        \State{\bf (Generating a Laplacian random matrix)}  Generate i.i.d. independent random variables $\l_{ij}\sim\Lap(\sigma), i\leq j$.
        Define a symmetric matrix $\bA$ such that 
        \[\bA_{ij} =  \bA_{ji}= \left\{\begin{aligned}
            &\lambda_{ij} \quad& \textrm{if }i<j;\\
            &2 \lambda_{ii} &\quad \textrm{if }i=j,
        \end{aligned}\right.\]
    
    \State {\bf Output:} The noisy covariance matrix $\widehat{\bM} = \bM+\bA$.
    \end{algorithmic}
\end{algorithm}

We start with the first step: finding a $d'$ dimensional private linear affine subspace and projecting $\bX$ onto it. Consider the $d\times n$ data matrix $\bX=[X_1,\dots,X_n]$, where $X_1,\dots, X_n\in \mathbb R^d$. The rank of the covariance matrix $\frac{1}{n}\bX\bX^\top$ measures the dimension of the \textit{linear subspace} spanned by $X_1,\dots, X_n$.  If  we  subtract the mean vector  and consider the  centered covariance matrix $\bM$ in \eqref{eq: sample_cov}, 
then the rank of $\bM$ indicates the dimension of the \textit{affine linear subspace} that $\bX$ lives in.

To guarantee the privacy of $\bM$, we add a symmetric Laplacian random matrix $\bA$ to $\bM$ to create a private Hermitian matrix $\widehat{\bM}$ from Algorithm \ref{alg: covariance}. The variance of entries in $\bA$ is chosen such that the following privacy guarantee holds:

\begin{proposition}
\label{prop: privacy_covariance}  Algorithm~\ref{alg: covariance} is $\e$-differentially private.
\end{proposition}

\begin{proof}[Proof]
    Before applying the definition of differential privacy, we compute the entries of $\bM$ explicitly. One can easily check that
    \begin{equation}
    \label{eq: sample_cov_expansion}
       \bM = \frac{1}{n}\sum_{k=1}^n X_kX_k^\top - \frac{1}{n(n-1)}\sum_{k\neq \ell}X_kX_\ell^\top.
    \end{equation}

    Now, if there are neighboring datasets $\bX$ and $\bX'$, suppose $X_k = (X_k^{(1)},\dots,X_k^{(d)})^\top$ is a column vector in $\bX$ and $X_k' = ({X_k'}^{(1)},\dots,{X_k'}^{(d)})^\top$ is a column vector in $\bX'$, and all other column vectors are the same. Let $\bM$ and $\bM'$ be the covariance matrix of $\bX$ and $\bX'$, respectively. Then we consider the density function ratio for the output of  Algorithm~\ref{alg: covariance} with input $\bX$ and $\bX'$:
    \begin{align*}
        \frac{\den_{A}(\widehat{\bM}-\bM)}{\den_{A}(\widehat{\bM}-\bM')} &= \prod_{i<j}\frac{\den_{\lambda_{ij}}((\widehat{\bM}-\bM)_{ij})}{\den_{\lambda_{ij}}((\widehat{\bM}-\bM' )_{ij})} \prod_{i=j}\frac{\den_{2\lambda_{ij}}((\widehat{\bM}-\bM)_{ij})}{\den_{2\lambda_{ij}}((\widehat{\bM}-\bM' )_{ij})}\\
        &= \prod_{i<j} \frac{\exp\Big(-\frac{|(\widehat{\bM}-\bM )_{ij}|}{\s}\Big)}{\exp\Big(-\frac{|(\widehat{\bM}-\bM' )_{ij}|}{\s}\Big)} \prod_{i}\frac{\exp\Big(-\frac{|(\widehat{\bM}-\bM )_{ii}|}{2\s}\Big)}{\exp\Big(-\frac{|(\widehat{\bM}-\bM' )_{ii}|}{2\s}\Big)}\\
        &\leq \exp\left(\sum_{i<j} \abs{\bM_{ij}-\bM_{ij}'}/\s + \sum_{i} \abs{\bM_{ii}-\bM_{ii}'}/(2\s)\right)\\
        &= \exp\left(\frac{1}{2\s}\sum_{i,j}\abs{\bM_{ij}-\bM_{ij}'}\right).
    \end{align*}
    As the datasets differs on only one data $X_k$, consider all entry containing $X_k$ in \eqref{eq: sample_cov_expansion}, we have
    \begin{align*}
    \abs{\bM_{ij}-\bM'_{ij}} 
    \leq &\, \frac{1}{n}\abs{X_k^{(i)}X_k^{(j)}-{X_k'}^{(i)}{X_k'}^{(j)}} + \frac{1}{n(n-1)}\sum_{\ell\neq k}\abs{X_k^{(i)}-{X_k'}^{(i)}}{X_\ell}^{(j)} \\&+ \frac{1}{n(n-1)}\sum_{\ell\neq k}{X_\ell}^{(i)}\abs{X_k^{(j)}-{X_k'}^{(j)}}\\
    \leq &\, \frac{2}{n} + \frac{2}{n(n-1)}\cdot 2(n-1)\\
    =&\,\frac{6}{n}.
    \end{align*}
    Therefore, substituting the result in the probability ratio implies
    \[\frac{\den_{A}(\widehat{\bM}-\bM)}{\den_{A}(\widehat{\bM}-\bM')} \leq \exp\left(\frac{1}{2\s}\cdot d^2\cdot \frac{6}{n} \right) = \exp\left(\frac{3d^2}{\s n}\right),\]
    and when $\sigma = \frac{3 d^2}{\e n}$, Algorithm~\ref{alg: covariance} is $\e$-differentially private.
\end{proof}



\subsection{Noisy projection}

The private covariance matrix $\widehat{\bM}$ induces private subspaces spanned by eigenvectors of $\widehat{\bM}$.
We then perform a truncated SVD on $\widehat{\bM}$ to find a private $d'$-dimensional subspace $\widehat{\mathbf V}_{d'}$ and project original data onto $\widehat{\mathbf V}_{d'}$. Here, the matrix $\widehat{\mathbf V}_{d'}$ also indicates the subspace generated by its orthonormal columns. The full steps are summarized in Algorithm \ref{alg: projection}. 

\begin{algorithm}	
    \caption{Noisy Projection} \label{alg: projection}
    \begin{algorithmic}
    \State {\bf Input:}  True data matrix $\bX=[X_1,\dots,X_n]$, $X_i\in [0,1]^d$, privacy parameters $\e$, the private covariance matrix $\widehat{\bM}$ from Algorithm \ref{alg: covariance}, and a target dimension $d'$.

        \State{\bf (Singular value decomposition)} Compute the top $d'$ orthonormal eigenvectors $\widehat v_1,\dots, \widehat v_{d'}$ of $\widehat \bM$ and denote $\widehat{\bV}_{d'}=[\widehat v_1,\dots, \widehat v_{d'}]$.

        \State{\bf (Private centering)} Compute $\overline X = \frac{1}{n}\sum_{i=1}^n X_i$. Let $\l\in \mathbb R^d$ be a random vector with  i.i.d. components of   $\Lap(d/(\e n))$. Shift each $X_i$ to $X_i-(\overline X + \l)$ for $i\in [n]$.
        
        \State{\bf (Projection)} Project $\{X_i-(\overline X + \l )\}_{i=1}^n$ onto the linear subspace spanned by  $\widehat v_1,\dots, \widehat v_{d'}$. The projected data $\widehat{X}_i$ is given by
        $\widehat{X}_i = \sum_{j=1}^{d'} \<X_i-(\overline X + \l),\widehat v_j\>\widehat v_j.$

    \State {\bf Output:} The data matrix after projection $\widehat{\bX}=[\widehat{X}_1\dots \widehat{X}_n]$.
    \end{algorithmic}
\end{algorithm}

 Algorithm \ref{alg: projection} only guarantees  private basis $\widehat v_1,\dots, \widehat v_{d'}$ for each $\widehat{X}_i$, but the coordinates of $\widehat{X}_i$ in terms of $\widehat v_1,\dots, \widehat v_{d'}$ are \textit{not private}. Algorithms \ref{alg: pmm} and \ref{alg: psmm} 
 in the next stage will output synthetic data on the private subspace $\widehat{\mathbf V}_{d'}$ based on $\widehat{\bX}$. The privacy analysis combines the two stages based on Lemma \ref{lem: composition}, and we state the results in Section~\ref{sec: subroutines}.

\subsection{Accuracy guarantee for noisy projection}

The data matrix $\widehat{\bX}$ corresponds to an empirical measure $\mu_{\widehat{\bX}}$ supported on the subspace $\widehat{\mathbf V}_d$. In this subsection, we characterize the 1-Wasserstein distance between the empirical measure $\mu_{\widehat{\bX}}$ and the empirical measure of the centered dataset $\bX- \overline{X} \mathbf{1}^\top$, where $\mathbf{1}\in \mathbb R^n$ is the all-1 vector.  This problem can be formulated as the stability of a low-rank projection based on a covariance matrix with  additive noise. We first provide  the following useful deterministic lemma.


\begin{lemma}[Stability of  noisy projection]\label{lem:SVD}
Let $\bX$ be a $d\times n$ matrix and $\bA$ be a $d\times d$ Hermitian matrix. Let $\bM=\frac{1}{n}\bX\bX^\top$ with eigenvalues $\s_1\geq \s_2\geq \dots\geq\s_d$. Let $\widehat{\bM}=\frac{1}{n} \bX\bX^\top +\bA$, $\widehat{\bV}_{d'}$ be a $d\times d'$ matrix whose columns are the first $d'$ orthonormal eigenvectors of $\widehat{\bM}$, and $\bY=\widehat{\bV}_{d'} \widehat{\bV}_{d'}^\top \bX$. Let $\mu_{\bX}$ and $\mu_{\bY}$ be the empirical measures of column vectors of $\bX$ and $\bY$, respectively. Then 
\begin{align}\label{eq:Frobinus_error_proj}
     W_{2}^2(\mu_{\bX},\mu_{\bY})\leq\frac{1}{n}\|\bX-\bY\|_F^2\leq \sum_{i>d'}\sigma_i +  2 d'\|\bA\|.
\end{align}
\end{lemma}

\begin{proof}
   Let $\widehat v_1,\dots,\widehat v_{d}$ be a set of  orthonormal eigenvectors for $\widehat\bM$ with the corresponding eigenvalues $\widehat{\sigma}_1,\dots, \widehat{\sigma}_d$. Define four  matrices whose column vectors are eigenvectors:
    \begin{align*}
    \bV&= [v_1,\dots,v_{d}], \quad\quad \widehat \bV = [\widehat v_1,\dots,\widehat v_{d}],\\
\bV_{d'} &= [v_1,\dots,v_{d'}], \quad\quad \widehat \bV_{d'} = [\widehat v_1,\dots,\widehat v_{d'}].
    \end{align*}
    
    By orthogonality,  the following identities hold:
    \begin{align*}
    \sum_{i=1}^d \|v_i^\top \bX\|_2^2&=\sum_{i=1}^d \|\widehat v_i^\top \bX\|_2^2=\|\bX\|_F^2,\\
    \sum_{i>d'} \|v_i^\top \bX\|_2^2&=\|\bX-\bV_{d'} \bV_{d'}^\top \bX\|_F^2,\\
    \sum_{i>d'} \|\widehat v_i^\top \bX\|_2^2&=\|\bX-\widehat{\bV}_{d'} \widehat{\bV}_{d'}^\top \bX\|_F^2.
    \end{align*}
Separating the top $d'$ eigenvectors from the rest, we obtain
    \begin{align*}
       \sum_{i\leq d'} \|v_i^\top \bX\|_2^2 +\|\bX-{\bV}_{d'} {\bV}_{d'}^\top \bX\|_F^2=\sum_{i\leq d'}\|\widehat v_i^\top \bX\|_2^2+\|\bX-\widehat{\bV}_{d'} \widehat{\bV}_{d'}^\top \bX\|_F^2.
    \end{align*}
    Therefore 
    \begin{align}
         \|\bX-\widehat{\bV}_{d'} \widehat{\bV}_{d'}^\top \bX\|_F^2&=\sum_{i\leq d'} \|v_i^\top \bX\|_2^2-\sum_{i\leq d'}\|\widehat v_i^\top \bX\|_2^2 +\|\bX-{\bV}_{d'} {\bV}_{d'}^\top \bX\|_F^2 \nonumber\\
        &=n\sum_{i\leq d'}\sigma_i-n\sum_{i\leq d'}\widehat{v}_i^\top\bM\widehat{v}_i+ n\sum_{i>d'}\sigma_i \nonumber\\
        &=n\sum_{i\leq d'}\sigma_i-n\sum_{i\leq d'}\widehat{v}_i^\top (\widehat\bM-\bA)\widehat{v}_i+ n\sum_{i>d'}\sigma_i \nonumber\\
        &=n\sum_{i\leq d'}(\sigma_i-\widehat{\s}_i) +n\tr(\bA\widehat{\bV}_{d'} \widehat{\bV}_{d'}^\top )+n\sum_{i>d'}\sigma_i.\label{eq:truncated_SVD_error}
    \end{align}
    
    By Weyl's inequality, for $i\leq d'$,
    \begin{align}\label{eq:Weyl}
       |\sigma_i-\widehat{\s}_i| \leq \|\bA\|. 
    \end{align} 
    
    By von Neumann’s trace inequality,
    \begin{align}\label{eq:Von}
     \tr(A\widehat{\bV}_{d'} \widehat{\bV}_{d'}^\top ) \leq \sum_{i=1}^{d'} \sigma_i(\bA).   
    \end{align} 
 From \eqref{eq:truncated_SVD_error}, \eqref{eq:Weyl}, and \eqref{eq:Von}, 
    \[ \frac{1}{n} \|\bX-\widehat{\bV}_{d'} \widehat{\bV}_{d'}^\top \bX\|_F^2\leq \sum_{i>d'}\sigma_i + d'\|\bA\|+\sum_{i=1}^{d'} \sigma_i(\bA) \leq \sum_{i>d'}\sigma_i + 2 d'\|\bA\|.\]
Let $Y_i$ be the $i$-th column of $\bY$. We have 
\begin{align}
W_{2}^2(\mu_{\bX},\mu_{\bY})\leq  \frac{1}{n} \sum_{i=1}^n \|X_i-Y_i\|_2^2  =\frac{1}{n} \|\bX-\bY\|_F^2.
\end{align}
    Therefore  \eqref{eq:Frobinus_error_proj} holds.
\end{proof}

Note that inequality \eqref{eq:Frobinus_error_proj} holds without any spectral gap assumption on $\bM$. \add{Applying  Davis-Kahan inequality would require $\sigma_{d'}-\sigma_{d'+1}$ to be large while Lemma~\ref{lem:SVD} is applicable even when $\sigma_{d'}=\sigma_{d'+1}$.}
In the context of sample covariance matrices for random datasets, a related bound without a spectral gap condition is derived in  \cite[Proposition 2.2]{reiss2020nonasymptotic}.  Furthermore, Lemma~\ref{lem:SVD} bears a conceptual resemblance to \cite[Theorem 5]{achlioptas2001fast}, which deals with low-rank matrix approximation under perturbation.
With Lemma~\ref{lem:SVD}, we derive the following Wasserstein distance bounds between the centered dataset $\bX- \overline{X} \mathbf{1}^\top$ and the dataset $\widehat{\bX}$.
\begin{proposition}
    \label{thm: accuracy_covariance}
    For input data $\bX$  and output data $\widehat \bX$ in Algorithm~\ref{alg: projection}, let  $\bM$ be the covariance matrix defined in \eqref{eq: sample_cov}. \add{Assume $n\geq 1/\e$.} Then for an absolute constant $C>0$,
    \[ \E W_1(\mu_{\bX-\overline X \1^\top},\mu_{\widehat \bX})\leq \left(\E W_2^2(\mu_{\bX-\overline X \1^\top},\mu_{\widehat \bX})\right)^{1/2}\leq \sqrt{2\sum_{i> d'}\sigma_i(\bM)}+\sqrt{\frac{C d' d^{2.5}}{\e n}}.\]
\end{proposition}

\begin{proof}
    For the true covariance matrix $\bM$, consider its SVD
\begin{align}\label{eq:M_expression}
    {\bM} = \frac{1}{n-1}\sum_{i=1}^n (X_i - \overline{X})(X_i-\overline X)^\top = \sum_{j=1}^d \sigma_j  v_j  v_j^\top,
    \end{align}
    where $\s_1\geq \s_2\geq \dots\geq\s_d$ are the singular values and $ v_1\dots v_d$ are corresponding orthonormal eigenvectors. Moreover, define two $d\times d'$ matrices 
    \[\bV_{d'} = [v_1,\dots,v_{d'}], \quad\quad \widehat \bV_{d'} = [\widehat v_1,\dots,\widehat v_{d'}].\]
    Then the matrix 
    $\widehat{\bV}_{d'} \widehat{\bV}_{d'}^\top $ is a projection onto the subspace spanned by the principal components $\widehat v_1,\dots,\widehat v_{d'}$. 

    In Algorithm~\ref{alg: projection}, for any data $X_i$ we first shift it to $X_i-\overline X - \l$ and then project it to $\widehat{\bV}_{d'} \widehat{\bV}_{d'}^\top (X_i-\overline X - \l)$. Therefore
    \begin{align*}
        \norm{X_i-\overline X - \widehat{\bV}_{d'} \widehat{\bV}_{d'}^\top (X_i-\overline X - \l)}_\infty &\leq \norm{ X_i -\overline X - \widehat{\bV}_{d'} \widehat{\bV}_{d'}^\top (X_i-\overline X)}_\infty + \norm{\widehat{\bV}_{d'} \widehat{\bV}_{d'}^\top \l}_\infty\\
        &\leq \norm{ X_i -\overline X - \widehat{\bV}_{d'} \widehat{\bV}_{d'}^\top (X_i-\overline X)}_2 + \norm{\l}_2.
    \end{align*}
    
    Let $Z_i$ denote $X_i-\overline X$ and $\bZ = [Z_1,\dots,Z_n]$. Then 
    \[\frac{1}{n}\bZ\bZ^\top = \frac{n-1}{n} \bM.\]
    With Lemma~\ref{lem:SVD}, by definition of the Wasserstein distance, we have
    \begin{align}
        W_2^2(\mu_{\bX-\overline X \1^\top},\mu_{\widehat \bX})& \add{\leq} \frac{1}{n}\sum_{i=1}^n\norm{X_i-\overline X - \widehat{\bV}_{d'} \widehat{\bV}_{d'}^\top (X_i-\overline X - \l)}_\infty^2 \\
        &\leq \frac{2}{n}\sum_{i=1}^n \norm{ X_i -\overline X - \widehat{\bV}_{d'} \widehat{\bV}_{d'}^\top (X_i-\overline X)}_2^2 +  2\norm{\l}_2^2\\
        &=\frac{2}{n} \|\bZ-\widehat{\bV}_{d'} \widehat{\bV}_{d'}^\top \bZ\|_F^2 + 2\norm{\l}_2^2\\
        &\leq 2\sum_{i=d'}^n \s_i(\bM) + 4d'\norm{\bA} + 2\norm{\l}_2^2. \label{eq:W2_three_terms}
    \end{align}
    
   Since $\l=(\l_1,\dots,\l_d)$ is a Laplacian random vector with i.i.d. $\Lap(1/(\e n))$ entries,
    \begin{equation}
    \label{eq: laplacian_2_norm}
        \E \norm{\l}_2 ^2 = \sum_{j=1}^d \E\abs{\l_j}^2 = \frac{2d}{\e^2 n^2}.
    \end{equation}

    Furthermore, in Algorithm~\ref{alg: covariance}, $A$ is a symmetric random matrix with independent Laplacian random variables on and above its diagonal. Thus, we have the tail bound for its norm  \cite[Theorem 1.1]{dai2022tail}
    \begin{equation}
    \label{eq: matrix_tail_bound}
        \Pr{\norm \bA\geq \s(C\sqrt d + t)}\leq C_0 \exp(-C_1\min(t^2/4, t/2)).
    \end{equation}
    And we can further compute the expectation bound for $\norm{\bA}$ from \eqref{eq: matrix_tail_bound} with the choice of $\s = \frac{3d^2}{\e n}$,
    \begin{equation}
    \label{eq: matrix_E_bound}
        \E\|\bA\|\leq C\s\sqrt{d} + \int_{0}^\infty C_0 \exp\Big(\!\!- C_1\min\Big(\frac{t^2}{4\s^2},\frac{t}{2\s}\Big)\Big)\dd t \lesssim \frac{d^{2.5}}{\e n}.
    \end{equation}


    Combining the two bounds above and \eqref{eq:W2_three_terms}, as the 1-Wasserstein distance is bounded by the 2-Wasserstein distance and inequality $\sqrt{x+y}\leq \sqrt{x}+\sqrt{y}$ holds for all $x,y\geq 0$,
    \begin{align*}
        \E W_1(\mu_{\bX-\overline X \1^\top},\mu_{\widehat \bX}) &\leq \left(\E W_2^2(\mu_{\bX-\overline X \1^\top},\mu_{\widehat \bX})\right)^{1/2}\\
        & \leq \sqrt{2\sum_{i>d'}\sigma_i(\bM)} + \sqrt{4d'\E\norm{\bA}} + \sqrt{2\E\norm{\l}_2^2}\\
        &\leq \sqrt{2\sum_{i> d'}\sigma_i(\bM)}+\sqrt{\frac{C d' d^{2.5}}{\e n}}, 
    \end{align*}
    \add{where the last inequality holds under the assumption $\e n \geq 1$. }
\end{proof}

\section{Synthetic data subroutines}
\label{sec: subroutines}

In the next stage of Algorithm~\ref{alg: affine}, we construct synthetic data on the private subspace $\widehat{\mathbf V}_{d'}$. Since  the original data $X_i$ is in  $[0,1]^d$, after Algorithm~\ref{alg: projection},  we have
 \begin{align}\label{eq:defR}
  \norm{\widehat{X}_i}_2=\norm{X_i - \overline X - \l}_2  \leq \sqrt{d} + \norm{\overline X + \l}_2 =: R   
 \end{align}
for any fixed $\l\in \mathbb R^d$.
Therefore, the data after projection would lie in a $d'$-dimensional ball embedded in $\R^d$ with radius $R$, and the domain for the subroutine is
$$\Omega' = \{a_1\widehat v_1+\dots+a_{d'}\widehat v_{d'}\mid a_1^2+\dots+a_{d'}^2\leq R^2\},$$
where $\widehat v_1, \dots,\widehat v_{d'}$ are the first $d'$ private principal components in Algorithm~\ref{alg: projection}. 
Depending on whether $d'=2$ or $d'\geq 3$, we apply two different algorithms from \cite{he2023algorithmically}:  \add{private measure mechanism (PMM) and  private signed measure mechanism (PSMM)}. 

\add{A major difference between the two methods is the partition step. PMM uses a hierarchical binary partition of the entire space into $r$ layers, while PSMM partitions the entire space into disjoint regions. When $d'=2$, PMM has a better accuracy rate while when $d'\geq 3$, PSMM has a better dependence on $d'$ in the accuracy bound; See Remark~\ref{rmk:compare} for more details.}

\subsection{$d'=2$: private measure mechanism (PMM)}
The synthetic data subroutine Algorithm \ref{alg: pmm} is adapted from the Private Measure Mechanism (PMM) in \cite[Algorithm~4]{he2023algorithmically}.
The PMM algorithm generates synthetic data in a hypercube by first partition the cube and then perturb the count in each sub-regions. It involves a certain partition structure, \textit{binary hierarchical partition}.

\begin{definition}[Binary hierarchical partition, \cite{he2023algorithmically}]
    A \textit{binary hierarchical partition} of a set $\Omega$ of depth $r$ 
    is a family of subsets $\Omega_\theta$ indexed by 
    $\theta \in \{0,1\}^{\le r}$, where 
    $$
    \{0,1\}^{\le k} = \{0,1\}^0 \sqcup \{0,1\}^1 \sqcup \cdots \sqcup \{0,1\}^k, 
    \quad k=0,1,2\dots, 
    $$
    and such that $\Omega_{\theta}$ is partitioned into $\Omega_{\theta0}$ and $\Omega_{\theta 1}$ for every $\theta \in \{0,1\}^{\le r-1}$. By convention, the cube $\{0,1\}^0$ corresponds to $\emptyset$ and we write $\Omega_\emptyset = \Omega$.
\end{definition}

The detailed description of Algorithm~\ref{alg: pmm} is as follows. The privacy and accuracy guarantees of Algorithm~\ref{alg: pmm} are proved in the next proposition after stating the algorithm.

For the new region $\Omega'$ where projected data located, we first enlarge this $\ell_2$-ball of radius $R$ into a hypercube $\Omega_{\mathrm{PMM}}$ of edge length $2R$ defined in Algorithm~\ref{alg: pmm}. Both the $\ell_2$-ball $\Omega'$ and the larger hypercube $\Omega_{\mathrm{PMM}}$ are inside the subspace $\widehat{\bV}_{d'}$. 

Next, for the hypercube $\Omega_{\mathrm{PMM}}$, we are going to run PMM in \cite{he2023algorithmically}. We obtain a binary hierarchical partition $\{\Omega_\theta\}_{\theta\in \{0,1\}^{\leq r}}$ for $r = \lceil\log_2(\e n)\rceil$ by doing equal divisions of the hypercube recursively for $r$ rounds. Each round after the division, we count the data points in every new subregion $\Omega_\theta$ and add integer Laplacian noise to it. 

Finally, a consistency step ensures the output is a well-defined probability measure. Here, the counts are considered to be \textit{consistent} if they are non-negative and the counts of two smaller subregions  $\Omega_{\theta 0}, \Omega_{\theta 1}$
can add up to the counts of the larger regions $\Omega_\theta$ containing them.
We refer the readers to \cite{he2023algorithmically} for more detailed procedures of this step.


\begin{algorithm}	
    \caption{PMM Subroutine} \label{alg: pmm}
    \begin{algorithmic}
    \State {\bf Input:} dataset $\widehat \bX= (\widehat X_1,\dots,\widehat X_n)$ in the region 
    \[\Omega' = \{a_1\widehat v_1+\dots+a_{d'}\widehat v_{d'}\mid a_1^2+\dots+a_{d'}^2\leq R\}.\]

        \State{\bf (Binary partition)} Let $r = \lceil\log_2(\e n)\rceil$ and $\s_j = \e^{-1}\cdot 2^{\frac{1}{2}(1-\frac{1}{d'})(r-j)}$. Enlarge the region $\Omega'$ into
        \[\Omega_{\textrm{PMM}} = \{a_1\widehat v_1+\dots+a_{d'}\widehat v_{d'} \mid   a_i\in [-R,R], \forall i\in [d']\}.\]
        Build a binary partition $\{\Omega_\theta\}_{\theta\in \{0,1\}^{\leq r}}$ on $\Omega_{\textrm{PMM}}$. 
        \State{\bf (Noisy count)} For any $\theta$, count the number of data in the region $\Omega_{\theta}$ denoted by $n_\theta = \left|\widehat \bX\cap \Omega_\theta\right|$, and let $n_\theta'=(n_\theta+\l_\theta)_+$, where $\l_\theta$ are independent integer Laplacian random variables with $\l\sim \Lap_{\mathbb Z}(\sigma_{|\theta|})$, and $|\theta|$ is the length of the vector $\theta$.
        \State{\bf (Consistency)} Enforce consistency of $\{n_\theta'\}_{\theta\in\{0,1\}^{\leq r}}$.

    \State {\bf Output:} 
    \add{Synthetic data $\bX'$ generated by selecting $n_\theta'$ many data points arbitrarily (independently of $\widehat{\bX}$) from $\Omega_\theta$ for every $\theta\in \{0,1\}^r$.}
    \end{algorithmic}
\end{algorithm}

\begin{proposition}
\label{prop: expectation_pmm}
The subroutine Algorithm~\ref{alg: pmm} is $\e$-differentially private. \add{Assume $n\geq 1/\e$.} For any $d'\geq 2$, with the input as the projected data $\widehat \bX$ and the range $\Omega'$ with radius $R$, Algorithm~\ref{alg: pmm} has an accuracy bound
\[\E W_1(\mu_{\widehat \bX},\mu_{\bX'})\leq
CR (\e n)^{-1/d'},
\]
where the expectation is taken with respect to the randomness of the synthetic data subroutine, conditioned on $R$.
\end{proposition}

\begin{proof}
    The privacy guarantee follows from \cite[Theorem 1.1]{he2023algorithmically}. For accuracy, note that the region $\Omega'$ is a subregion of a $d'$-dimensional ball. Algorithm \ref{alg: pmm} enlarges the region to a $d'$-dimensional hypercube with side length $2R$. By re-scaling the size of the hypercube and applying  \cite[Corollary 4.4]{he2023algorithmically}, we obtain the accuracy bound.
\end{proof}

\subsection{$d'\geq 3$: private signed measure mechanism (PSMM)}
The Private Signed Measure Mechanism (PSMM) introduced in \cite{he2023algorithmically} generates a synthetic dataset $\bY$ in a compact domain $\Omega$ whose empirical measure  $\mu_{\bY}$ is close to the empirical measure $\mu_{\bX}$ of the original dataset $\bX$ under the 1-Wasserstein distance.

PSMM runs in polynomial time, and the main steps are as follows. We first partition the domain $\Omega$ into $m$ disjoint subregions $\Omega_1,\dots, \Omega_m$ and count the number of data points in each subregion. Then, we perturb the counts in each subregion with i.i.d. integer Laplacian noise. Based on the noisy counts, one can approximate $\mu_{\bX}$ with a signed measure $\nu$ supported on $m$ points. Then, we find the closest probability measure $\hat{\nu}$ to the signed measure $\nu$ under the bounded Lipschitz distance by solving a linear programming problem.

\begin{algorithm}	
    \caption{PSMM Subroutine} \label{alg: psmm}
    \begin{algorithmic}
    \State {\bf Input:} dataset $\widehat \bX= (\widehat X_1,\dots,\widehat X_n)$ in the region 
    \[\Omega' = \{a_1\widehat v_1+\dots+a_{d'}\widehat v_{d'}\mid a_1^2+\dots+a_{d'}^2\leq R^2\}.\]

        \State{\bf (Integer lattice)} Let $\d = \sqrt{d/d'}(\e n)^{-1/d'}$. Find the lattice over the region:
        \[L = \{a_1\widehat v_1+\dots+a_{d'}\widehat v_{d'}\mid a_1^2+\dots+a_{d'}^2\leq R^2, a_1,\dots,a_{d'} \in \delta \Z\}.\]
        \State{\bf (Counting)} For any $v=a_1\widehat v_1+\dots+a_{d'}\widehat v_{d'}\in L$, count the number
        \[n_v = \left|\widehat \bX\cap \{b_1\widehat v_1+\dots+b_{d'}\widehat v_{d'}\mid  b_i\in [a_i,a_i +\d)\}\right|.\]
        \State{\bf (Adding noise)} Define a synthetic signed measure $\nu$ such that for any $v\in L$,
        \[\nu(\{v\}) = (n_v + \lambda_v)/n,\]
        where $\l_v\sim \Lap_{\mathbb Z}(1/\e)$, $v\in L$ are i.i.d. random variables.
        \State{\bf (Synthetic probability measure)} Use linear programming and find the closest probability measure \add{$\widehat{\nu}$} to $\nu$ \add{under the bounded Lipschitz distance}.

    \State {\bf Output:} 
    \add{Synthetic data $\bX'$ containing copies of elements in $L$ so that $\mu_{\bX'}$ and $\widehat{\nu}$ are arbitrarily close (such $\bX'$ exist when the size of $\bX'$ is large enough; see \cite[Section 3]{he2023algorithmically})}.
    \end{algorithmic}
\end{algorithm}

We provide the main steps of PSMM in Algorithm \ref{alg: psmm}. Details about the linear programming in the  \textit{synthetic probability measure}  step can be found in \cite{he2023algorithmically}. We apply PSMM from \cite{he2023algorithmically} when the metric space is an $\ell_2$-ball of radius $R$ inside $\widehat{\mathbf V}_{d'}$ and the following privacy and accuracy guarantees hold:
\begin{proposition}
\label{prop: expectation_psmm}
    The subroutine Algorithm~\ref{alg: psmm} is $\e$-differentially private. \add{Assume $n\geq 1/\e$.}  When $d'\geq 3$, with the input as the projected data $\widehat \bX$ and the range $\Omega'$ with radius $R$,
    the algorithm has an accuracy bound
    \begin{align}\label{eq:prop5}
    \E W_1(\mu_{\widehat \bX},\mu_{\bX'})\lesssim \frac R{\sqrt{d'}}(\e n)^{-1/d'},
    \end{align}
    where the expectation is conditioned on $R$.
\end{proposition}
\begin{proof}
The proposition is a direct corollary to the result in \cite{he2023algorithmically}. The size of the scaled integer lattice $\delta \Z$ in the unit $d$-dimensional ball of radius $R$ is bounded by $(\frac{C}{\d R})^d$  for an absolute constant $C>0$ (see, for example, \cite[Claim 2.9]{feige2005spectral} and \cite[Proposition 3.7]{boedihardjo2022covariance}). Then, the number of subregions in Algorithm \ref{alg: psmm} is bounded by
\[|L| \leq  \left(\frac{R}{\sqrt{d'}}\cdot\frac C\d \right) ^{d'}.\]
By \cite[Theorem 3.6]{he2023algorithmically}, we have
\[\E W_1(\mu_{\widehat \bX},\mu_{\bX'})\leq \d + \frac{2}{\e n}\left(\frac{R}{\sqrt{d'}}\cdot\frac C\d \right) ^{d'} \!\cdot\frac{1}{d'}\Bigg(\left(\frac{R}{\sqrt{d'}}\cdot\frac C\d \right) ^{d'}\Bigg)^{-\frac{1}{d'}}.\]
Taking $\d = \frac{CR}{\sqrt{d'}}(\e n)^{-\frac{1}{d'}}$  concludes the proof.
\end{proof}


\begin{remark}[PMM vs PSMM for $d'\geq 2$] \label{rmk:compare}
For general $d'\geq 2$, PMM can still be applied, and  the accuracy bound becomes 
   $\E W_1(\mu_{\widehat \bX},\mu_{\bX'})\leq CR (\e n)^{-1/d'} $. 
  \add{Compared to ~\eqref{eq:prop5}, } a the accuracy bound from PMM is weaker by a factor of $
   \sqrt{d'}$. However, as shown in \cite{he2023algorithmically}, PMM has a running time linear in $n$ and $d$, which is more computationally efficient than PSMM given in Algorithm \ref{alg: psmm} with running time polynomial in $n,d$.
\end{remark}


\subsection{Adding a private mean vector and  metric projection}

After generating the private synthetic data, since we shift the data by its private mean before projection, we need to add another private mean vector back, which shifts the dataset $\widehat{\bX}$ to a new private affine subspace close to the original dataset $\bX$. The output data vectors in $\bX''$ (defined in Algorithm~\ref{alg: affine})  are not necessarily inside $[0,1]^d$.
The subsequent metric projection enforces all synthetic data inside $[0,1]^d$. Importantly, this post-processing step does not have privacy costs.  

After metric projection,  dataset $\bY$ from the output of Algorithm \ref{alg: affine} is close to an affine subspace, as shown in the next proposition. Notably, \eqref{eq:metric_error} shows that the metric projection step does not cause the largest accuracy loss among all subroutines.

\begin{proposition}[$\bY$ is close to an affine subspace]\label{prop:Y_is_close}
The function $f:\R^d \to [0,1]^d$ in Algorithm~\ref{alg: affine} is the metric projection to $[0,1]^d$ with respect to $\|\cdot\|_{\infty}$, and the accuracy error for the metric projection step in Algorithm~\ref{alg: affine} is dominated by the error of the previous steps:
   \begin{align}\label{eq:metric_error}
    W_1(\mu_\bY,\mu_{\bX''}) \leq  W_1(\mu_\bX, \mu_{\bX''}), 
    \end{align}
    where the dataset $\bX''$ defined in Algorithm \ref{alg: affine} is in a $d'$-dimensional affine subspace. 
\end{proposition}
\begin{proof}
    For the function $f$ defined in Algorithm~\ref{alg: affine}, we know $f(x)$ is the closest real number to $x$ in the region $[0,1]$ for any $x\in \R$. Furthermore, if $v\in \R^d$ is a vector, then $f(v)$ is the closest vector to $v$ in $[0,1]^d$ with respect to $\|\cdot\|_\infty$. Thus $f:\R^d\to [0,1]^d$ is indeed a metric projection to $[0,1]^d$.

    We first assume that the synthetic data $\bX''$ also has size $n$. Then for any column vector $X_i''$, we know that $Y_i = f(X_i'')$ is its closest vector in $[0,1]^d$ under the $\ell^{\infty}$ metric. For the data $X_1,X_2,\dots, X_n$, suppose that the solution to the optimal transportation problem for $W_1(\mu_{\bX},\mu_{\bX''})$ is to match $X_{\tau(i)}$ with $X_{i}''$, where $\tau$ is a permutation on $[n]$. Then
    \[W_1(\mu_\bY, \mu_{\bX''}) \leq \frac{1}{n}\sum_{i=1}^{n} \norm{Y_i - X_i''}_\infty \leq \frac{1}{n}\sum_{i=1}^{n} \norm{X_{\tau(i)} - X_i''}_\infty = W_1(\mu_\bX, \mu_{\bX''}).\]
    In general, if the synthetic dataset has $m$ data points and $m\neq n$, we can split the points and regard both the true dataset and synthetic dataset as of size $mn$, then it's easy to check that the inequality still holds. 
 \end{proof}

\section{Privacy and accuracy of Algorithm~\ref{alg: affine}} 
\label{sec: main theorem}

In this section, we summarize the privacy and accuracy guarantees of Algorithm \ref{alg: affine}. The privacy guarantee is proved by analyzing three parts of our algorithms: private mean, private linear subspace, and private data on an affine subspace. 
\begin{proposition}[Privacy]
\label{thm: privacy_affine}
    Algorithm~\ref{alg: affine} is $\e$-differentially private.
\end{proposition}

\begin{proof}
    We can decompose Algorithm~\ref{alg: affine} into the following steps:
    \begin{enumerate}
        \item $\cA_1(\bX) = \widehat\bM$ is to compute the private covariance matrix with Algorithm~\ref{alg: covariance}.
        \item $\cA_2(\bX) = \overline X + \l$ is to compute the private sample mean. 
        \item $\cA_3(\bX, y, \Sigma)$ for fixed $y$ and $\Sigma$, is to project the shifted data $\{X_i - y\}_{i=1}^n$ to the first $d'$ principal components of $\Sigma$ and apply a certain differentially private subroutine (we  choose $y$ and $\Sigma$ as the output of $\cA_2$ and $\cA_1$, respectively). This step outputs synthetic data $\bX'=(X_1',\dots, X_m')$ on a linear subspace.
        \item $\cA_4(\bX, \bX')$ is to shift the dataset to $\{X_i' + \overline{X}_{\textrm{priv}}\}_{i=1}^{m}$, \add{where $\overline{X}_{\textrm{priv}}$ is the private mean vector of the true data step computed by $\cA_2$}.
        \item Metric projection.
    \end{enumerate}
    It suffices to show that the data before metric projection has already been differentially private. We will need to apply Lemma~\ref{lem: composition} several times.
    
    With respect to the input $\bX$ while fixing other input variables, we know that  $\cA_1,\cA_2,\cA_3,\cA_4$ are all $\e/4$-differentially private. 
    Therefore, by using Lemma~\ref{lem: composition} iteratively,  the composition algorithm 
    \[\cA_4(\bX, \cA_3(\bX, \cA_2(\bX), \cA_1(\bX)))\]
    satisfies $\e$-differential privacy. Hence Algorithm~\ref{alg: affine} is $\e$-differentially private.
\end{proof}

The next theorem combines errors from linear projection, synthetic data subroutine using  PMM or PSMM, and the post-processing error from mean shift and metric projection.
\begin{proposition}[Accuracy]
\label{prop: accuracy_affine}
 For  any given $2\leq d'\leq d$ and $n\geq 1/\e$, the output data $\bY$ from  Algorithm~\ref{alg: affine} with the input data $\bX$ satisfies
    \begin{equation}
    \label{eq: three-term}
        \E W_1(\mu_{\bX},\mu_{\bY}) \lesssim \sqrt{\sum_{i> d'}\sigma_i(\bM)} +\sqrt{\frac{d' d^{2.5}}{\e n}}+ \sqrt{\frac{d}{d'}}(\e n)^{-1/d'},
    \end{equation}
    where $\bM$ denotes the covariance matrix in \eqref{eq: sample_cov}.
\end{proposition}
\begin{proof}
    \add{In the case of $n<1/\e$, we have $W_1(\mu_{\bX},\mu_{\bY}) \leq 1\leq (\e n)^{-1/d'}$. The result is trivial. We assume $n\geq 1/\e$ in the rest of the proof.}

    Similar to privacy analysis, we will decompose the algorithm into several steps. Suppose that
    \begin{enumerate}
        \item $\bX-(\overline X + \l)\1^\top$ denotes the shifted data $\{X_i -\overline{X} - \l\}_{i=1}^n$;
        \item $\widehat \bX$ is the data after projection to the private linear subspace;
        \item $\bX'$ is the output of the synthetic data subroutine in Section~\ref{sec: subroutines};
        \item $\bX'' = \bX'+(\overline X + \l')\1^\top$ denotes the data shifted back;
        \item $\cA(\bX)$ is the data after metric projection, which is the output of the whole algorithm.
    \end{enumerate}
        
    For the metric projection step, by Proposition~\ref{prop:Y_is_close}, we have that  
    \begin{align}
        W_1(\mu_\bX,\mu_{\cA(\bX)})&\leq W_1(\mu_\bX,\mu_{\bX''}) + W_1(\mu_{\bX''},\mu_{\cA(\bX)})\leq 2 W_1(\mu_\bX,\mu_{\bX''}). \label{eq:W1XMX}
    \end{align}
    
    Moreover, applying the triangle inequality of Wasserstein distance to the other steps of the algorithm,  we have
    \begin{align}
        W_1(\mu_\bX,\mu_{\bX''}) & =  W_1(\mu_{\bX-\overline X \1^\top}, \mu_{\bX'+\l'\1^\top})\\
         &\leq  W_1(\mu_{\bX-\overline X \1^\top}, \mu_{\widehat \bX}) + W_1(\mu_{\widehat \bX}, \mu_{\bX'}) + W_1(\mu_{\bX'}, \mu_{\bX'+\l'})\\
         &\leq W_1(\mu_{\bX-\overline X \1^\top}, \mu_{\widehat \bX}) + W_1(\mu_{\widehat \bX}, \mu_{\bX'}) + \norm{\l'}_\infty. \label{eq:W1XMX2}
    \end{align}
    Note that $W_1(\mu_{\bX-\overline X \1^\top}, \mu_{\widehat \bX})$ is the projection error we bound in Theorem~\ref{thm: accuracy_covariance} \add{with $n\geq 1/\e$}, and $W_1(\mu_{\widehat \bX}, \mu_{\bX'})$ is treated in  the accuracy analysis of subroutines in Section~\ref{sec: subroutines}. Moreover, we have
    \begin{align*}
        \E W_1(\mu_{\widehat \bX}, \mu_{\bX'}) &= \E_{R} \E_{\bX'}W_1(\mu_{\widehat \bX}, \mu_{\bX'}) \\
        &\leq \E_{R}  \frac{CR}{\sqrt{d'}}(\e n)^{-1/d'}\\
        &\leq \frac{C(2\sqrt{d} + \E\norm{\l}_2)}{\sqrt{d'}}(\e n)^{-1/d'}\\
        &\lesssim \sqrt{\frac{d}{d'}}(\e n)^{-1/d'}.
    \end{align*}

    Here in the last step we use $\E\norm{\l}_2\leq \frac{C\sqrt{d}}{\e n}$ in \eqref{eq: laplacian_2_norm}.    Since $\lambda'$ is a sub-exponential random vector,  the following  bound also holds for some absolute constant $C>0$:
    \begin{align}\label{eq:l_infty_norm}
        \E\norm{\l'}_\infty\leq \frac{C\log d}{\e n}.
    \end{align}
    
    Hence
    \begin{align}
        \E W_1(\mu_\bX,\mu_{\cA(\bX)})
        &\leq 2 \E W_1(\mu_\bX,\mu_{\bX'+(\overline X + \l')\1^\top})\\
        &\leq 2\E W_1(\mu_{\bX-\overline X \1^\top}, \mu_{\widehat \bX}) + 2\E W_1(\mu_{\widehat \bX}, \mu_{\bX'}) + 2\E\norm{\l'}_\infty\\
        &\leq 2\sqrt{2\sum_{i> d'}\sigma_i(\bM)}+2\sqrt{\frac{C d' d^{2.5}}{\e n}} + 2C\sqrt{\frac{d}{d'}}(\e n)^{-1/d'} + \frac{2 C \log d}{\e n} \label{eq:expectation_XMX}\\
        &\lesssim \sqrt{\sum_{i> d'}\sigma_i(\bM)} + \sqrt{\frac{d}{d'}}(\e n)^{-1/d'}+ \sqrt{\frac{d' d^{2.5}}{\e n}},
    \end{align}
    where the first inequality is from \eqref{eq:W1XMX}, the second inequality is from \eqref{eq:W1XMX2}, and the third inequality is due to Theorem~\ref{thm: accuracy_covariance}, Proposition~\ref{prop: expectation_pmm}, and Proposition~\ref{prop: expectation_psmm}.
\end{proof}

\add{\section{Adaptive and private choice of $d'$}
\label{sec: d'choice}
In our main Algorithm~\ref{alg: affine}, $d'$ is regarded as a fixed input hyper-parameter. In this section, we will show that it is possible to choose $d'$ privately without sacrificing accuracy. 
\begin{lemma}
    \label{lem: cov_matrix_bound}
    For $\bM$ and $\widehat{\bM}$ defined in Algorithm~\ref{alg: covariance}, with probability at least $1-C\exp(-c\sqrt{d})$, there is
    \[\abs{\sum_{i>d'}\sigma_i(\widehat{\bM}) - \sum_{i>d'}\sigma_i({\bM})}\lesssim \frac{(d-d')d^{2.5}}{\e n}.\]
\end{lemma}
\begin{proof}
    By Weyl's inequality, $\abs{\sigma_i(\widehat{\bM}) - \sigma_i({\bM})} \leq \|\bA\|$. Applying the \eqref{eq: matrix_tail_bound} of the noise $\bA$ implies the inequality in the lemma.
\end{proof}
Therefore, from Proposition~\ref{prop: accuracy_affine}, with probability at least $1-C\exp(-c\sqrt{d})$, we have the following accuracy bound 
\begin{align}
    W_1(\mu_{\bX},\mu_{\bY}) &\lesssim \sqrt{\sum_{i> d'}\sigma_i(\bM)} +\sqrt{\frac{d' d^{2.5}}{\e n}}+ \sqrt{\frac{d}{d'}}(\e n)^{-1/d'} \\
    &\lesssim \sqrt{\sum_{i> d'}\sigma_i(\widehat{\bM}) + \frac{(d-d')d^{2.5}}{\e n}} +\sqrt{\frac{d' d^{2.5}}{\e n}}+ \sqrt{\frac{d}{d'}}(\e n)^{-1/d'}\\
    &\lesssim \sqrt{\sum_{i> d'}\sigma_i(\widehat{\bM})}+ \sqrt{\frac{d}{d'}}(\e n)^{-1/d'} +\sqrt{\frac{d^{3.5}}{\e n}}.
\end{align}
Since the last term above is not related to $d'$, we can choose 
\[d'\coloneqq \argmin_{2\leq k\leq d} \Bigg(\sqrt{\sum_{i> k}\sigma_i(\widehat \bM)} + \sqrt{\frac{d}{k}}(\e n)^{-1/k}\Bigg).\]
The privacy of the choice of $d'$ is guaranteed as we only use the private covariance matrix.
}

\section{Near-optimal accuracy bound with additional assumptions when $d'=1$}
\label{sec: d'=1}
Our Proposition~\ref{prop: accuracy_affine} is not applicable to the case $d'=1$ because the projection error in Theorem~\ref{thm: accuracy_covariance} only has bound $O((\e n)^{-\frac{1}{2}})$, which does not match with the optimal synthetic data accuracy bound in \cite{boedihardjo2022private,he2023algorithmically}. We are able to improve the accuracy bound with an additional dependence on $\s_1({\bM})$ as follows:
\begin{theorem}
\label{thm: d'=1}
    When $d'=1$, consider Algorithm~\ref{alg: affine} with input data $\bX$, output data $\bY$, and the subroutine PMM in Algorithm \ref{alg: pmm}. Let $\bM$ be the covariance matrix defines as \eqref{eq: sample_cov}. Assume  $\s_1(\bM) >0$ and \add{$n\geq 1/\e$}, then
    \[\E W_1(\mu_\bX,\mu_{\bY}) \lesssim \sqrt{ \sum_{i>1}\sigma_i(\bM)} + \frac{d^3}{\sqrt{\s_1(\bM)}\e n} + \frac{\sqrt{d}\log^2(\e n)}{\e n}.\]
\end{theorem}

We start with the following lemma based on the Davis-Kahan theorem \cite{yu2015useful}.
\begin{lemma}
\label{lem: Davis-Kahan error}
Let $\bX$ be a $d\times n$ matrix and $\bA$ be an $d\times d$ Hermitian matrix. Let $\bM=\frac{1}{n}\bX\bX^\top$, with the SVD \[{\bM}  =\sum_{j=1}^d \sigma_j  v_j  v_j^\top,\]
    where $\s_1\geq \s_2\geq \dots\geq\s_d$ are the singular values of $\bM$ and $ v_1,\dots, v_d$ are corresponding orthonormal eigenvectors. Let $\widehat{\bM}=\frac{1}{n} \bX\bX^\top + \bA$ with orthonormal eigenvectors $\widehat{v}_1,\dots ,\widehat{v}_d$, where $\widehat{v}_1$ corresponds to the top singular value of $\widehat{\bM}$. When there exists a spectral gap $\s_{1} - \s_{2}= \delta>0$, we have
\[\frac{1}{n}\|\bX-\widehat{v}_1\widehat{v}_1^\top \bX\|_F^2 \leq 2\sum_{i>\add{1}} \s_i + \frac{8d'^2 }{n\delta^2} \norm \bA ^2 \norm{\bX}_F^2.\]
\end{lemma}

\begin{proof}
    We have that
    \begin{align}
    \label{eq: vv-vv}
        \frac{1}{n}\|\bX-\widehat{v}_1 \widehat{v}_1^\top \bX\|_F^2 &= \frac{1}{n}\|\bX - v_1 v_1^\top \bX + v_1 v_1^\top \bX-\widehat{v}_1 \widehat{v}_1^\top \bX\|_F^2 \nonumber\\
        &\leq \frac{2}{n}\left(\|\bX - v_1 v_1^\top \bX\|_F^2 + \|v_1 v_1^\top \bX-\widehat{v}_1 \widehat{v}_1^\top \bX\|_F^2\right)\nonumber\\
        & = 2\sum_{i>\add{1}} \s_i + \frac{2}{n}\norm{\left(v_1 v_1^\top -\widehat{v}_1 \widehat{v}_1^\top\right) \bX}_F^2\nonumber\\
        & \leq 2\sum_{i>\add{1}} \s_i + \frac{2}{n}\norm{v_1 v_1^\top -\widehat{v}_1 \widehat{v}_1^\top }^2 \norm{\bX}_F^2.
    \end{align}

    To bound the operator norm distance between the two projections, we will need the Davis-Kahan Theorem in the perturbation theory. For the angle $\Theta(v_1, \widehat{v}_1)$ between the vectors $v_1$ and $\widehat{v}_1$, applying \cite[Corollary 1]{yu2015useful}, we have
    \[\norm{v_1 v_1^\top  - \widehat v_1\widehat v_1^\top} = {\sin \Theta(v_1, \widehat{v}_1)}\leq \frac{2\|{\bM-\widehat \bM}\|}{\sigma_{1} - \sigma_{2}} \leq \frac{2\norm \bA}{\d}.\]
    Therefore, when the spectral gap exists ($\delta>0$), 
    \[\frac{1}{n}\|\bX-\widehat{v}_1 \widehat{v}_1^\top \bX\|_F^2 \leq 2\sum_{i>\add{1}} \s_i + \frac{8}{n\delta^2} \norm \bA ^2 \norm{\bX}_F^2.\]
    This finishes the proof.
\end{proof}

Compared to Lemma \ref{lem:SVD}, with the extra spectral gap assumption, the dependence on $\bA$ in the upper bound changes from $\|\bA\|$ to $\|\bA\|^2$. A similar phenomenon, called \textit{global and local bounds}, was observed in \cite[Proposition 2.2]{reiss2020nonasymptotic}.  With Lemma \ref{lem: Davis-Kahan error}, we are able to improve the accuracy rate for the noisy projection step as follows.

\begin{proposition}
\label{thm: d'=1 projection}
 Let $\s_1\geq \dots\geq \s_d \geq 0$ be the singular values of $\bM$ defined in \eqref{eq:M_expression}.
 When $d'=1$, assume that $\s_1 >0$ \add{and $n\geq 1/\e$}. For the output $\widehat\bX$ in Algorithm~\ref{alg: projection}, we have
    \[ \E W_1(\mu_{\bX-\overline X \1^\top},\mu_{\widehat \bX})\leq \left(\E W_2^2(\mu_{\bX-\overline X \1^\top},\mu_{\widehat \bX})\right)^{1/2}\lesssim \sqrt{\sum_{i> 1} \sigma_i}+\frac{d^3}{\sqrt{\s_1}\e n},\]
   
\end{proposition}
\begin{proof}
    Similar to the proof of Theorem~\ref{thm: accuracy_covariance}, we can define $Z_i = X_i - \overline{X}$ and deduce that
    \begin{align*}
    \frac{1}{n} \bZ\bZ^\top &= \frac{n-1}{n} \bM,\\
    \frac{1}{n} \|\bZ\|_F^2 &= \frac{n-1}{n} \tr (\bM),
    \end{align*}
and    
    \[W_2^2(\mu_{\bX-\overline X \1^\top},\mu_{\widehat \bX})=\frac{2}{n} \|\bZ-\widehat{v}_1 \widehat{v}_1^\top \bZ\|_F^2 + 2\norm{\l}_2^2.\]
    By the inequality $\sqrt{x+y}\leq \sqrt{x} + \sqrt{y}$ for $x,y\ge 0$,
    \[\E W_1(\mu_{\bX-\overline X \1^\top},\mu_{\widehat \bX})\leq \E\left[\frac{2}{n} \|\bZ-\widehat{v}_1 \widehat{v}_1^\top \bZ\|_F^2\right]^{1/2} + \sqrt{2}\E \norm{\l}_2.\]
   Let $\delta = \s_1-\s_2$. Next, we will discuss two cases for the value of $\d$.

    \noindent \textbf{Case 1:} When $\d = \s_1 - \s_2 \leq \frac{1}{2}\s_1$, we have $\s_1\leq 2\s_2$ and
    \begin{align*}
         \tr(\bM) = \s_1 + \dots + \s_d \leq 3 \sum_{i>1} \s_i.
    \end{align*}
    As any projection map has spectral norm  1, we have $\norm{v_1 v_1^\top -\widehat v_1 \widehat v_1^\top }\leq 2$. Applying \eqref{eq: vv-vv}, we have
    \begin{align*}
        \frac{1}{n} \|\bZ-\widehat v_1 \widehat v_1^\top \bZ\|_F^2 &\leq 2\sum_{i>1} \s_i + \frac{2}{n}\norm{ v_1  v_1^\top -\widehat v_1 \widehat v_1^\top }^2 \norm{\bZ}_F^2 \\
        &\leq 2\sum_{i>1} \s_i + \frac{8}{n}\norm{\bZ}_F^2\\
        &\leq  2\sum_{i>1} \s_i + 8\tr(\bM)\leq 26 \sum_{i>1} \s_i.
    \end{align*}
    Hence 
    \begin{align}\label{eq:case1}
    \E W_1(\mu_{\bX-\overline X \1^\top},\mu_{\widehat \bX}) \lesssim \sqrt{\sum_{i>1} \s_i} + \E \norm{\l}_2 \lesssim \sqrt{\sum_{i>1} \s_i} + \frac{\sqrt{d}}{\e n}.
    \end{align}

    \noindent \textbf{Case 2:} When $\delta \geq \frac{1}{2}\s_1$, we have 
    \[\tr(\bM) \leq  d \s_1 \leq \frac{4d\d^2}{\s_1}.\]
    
    For any fixed $\d$, by Lemma~\ref{lem: Davis-Kahan error}, 
    \begin{align*}
        \frac{1}{n} \|\bZ-\widehat v_1 \widehat v_1^\top \bZ\|_F^2 &\leq 2\sum_{i>1} \s_i + \frac{8}{n\d^2}\norm{\bA}^2\norm{\bZ}_F^2\\
        &\leq 2\sum_{i>1} \s_i + \frac{8}{\d^2}\norm{\bA}^2 \tr(\bM) \\
        &\leq 2\sum_{i>1} \s_i +\frac{32d}{\s_1}\norm{\bA}^2.
    \end{align*}
    So we have the Wasserstein distance bound
    \begin{align}
        \E W_1(\mu_{\bX-\overline X \1^\top},\mu_{\widehat \bX}) &\leq \sqrt{2\sum_{i>1}\sigma_i} + \sqrt{\frac{32d}{\s_1}}\E\norm{\bA}  + \sqrt{2}\E\norm{\l}_2\\
        &\leq \sqrt{2\sum_{i>1}\sigma_i} + \sqrt{\frac{32d}{\s_1}} \frac{d^{2.5}}{\e n} + \frac{\sqrt{2d}}{\e n}\\
        &\leq \sqrt{2\sum_{i>1}\sigma_i} + \frac{C d^3}{\sqrt{\s_1}\e n}.\label{eq:case2}
    \end{align}
     From \eqref{eq:M_expression},
    \[ \sigma_1=\|M\|\leq \|M\|_F\leq \frac{n}{n-1}d\leq 2d. \]
    Combining the two cases \eqref{eq:case1} and \eqref{eq:case2}, we deduce the result.
\end{proof}

\begin{proof}[Proof of Theorem~\ref{thm: d'=1}]
    Following the steps in the  proof of Theorem~\ref{thm: accuracy_covariance}, we obtain
    \begin{align*}
        \E W_1(\mu_\bX,\mu_{\cA(\bX)}) &\leq 2 \E W_1(\mu_\bX,\mu_{\bX'+(\overline X + \l')\1^\top})\\ 
        &\leq 2\E W_1(\mu_{\bX-\overline X \1^\top}, \mu_{\widehat \bX}) + 2\E W_1(\mu_{\widehat \bX}, \mu_{\bX'}) + 2\E\norm{\l'}_\infty\\
        &\lesssim \sqrt{ \sum_{i>1}\sigma_i} + \frac{d'd^3}{\sqrt{\s_1}\e n} + \frac{\sqrt{d}\log^2(\e n)}{\e n} + \frac{2 C \log d}{\e n}\\
        &\lesssim \sqrt{ \sum_{i>1}\sigma_i} + \frac{d'd^3}{\sqrt{\s_1}\e n} + \frac{\sqrt{d}\log^2(\e n)}{\e n},
    \end{align*}
    where for the second inequality,  we apply the bound from \cite[Theorem 1.1]{he2023algorithmically} for the second term, and we use \eqref{eq:l_infty_norm} for the third term.
\end{proof}

\section{Conclusion}
In this paper, we provide a DP algorithm to generate synthetic data, which closely approximates the true data in the hypercube $[0,1]^d$ under 1-Wasserstein distance. Moreover, when the true data lies in a  $d'$-dimensional affine subspace, we improve the accuracy guarantees in \cite{he2023algorithmically} and circumvents the curse of dimensionality by generating a synthetic dataset close to the affine subspace.  

It remains open to determine the optimal dependence on $d$ in the accuracy bound in Proposition \ref{prop: accuracy_affine} and whether the third term in \eqref{eq: three-term} is needed.
Our analysis of private PCA works without using the classical Davis-Kahan inequality that requires a spectral gap on the dataset. However, to approximate a dataset close to a line ($d'=1$), additional assumptions are needed in our analysis to achieve the near-optimal accuracy rate, see Section~\ref{sec: d'=1}. It is an interesting problem to achieve an optimal rate without the dependence on $\sigma_1(\bM)$ when $d'=1$. 

Our Algorithm \ref{alg: affine} only outputs synthetic data with a low-dimensional linear structure, and its analysis heavily relies on linear algebra tools. For original datasets from a $d'$-dimensional linear subspace, we improve the accuracy rate from $(\e n)^{-1/(d'+1)}$ in \cite{donhauser2023sample} to $(\e n)^{-1/d'}$. It is also interesting to provide algorithms with optimal accuracy rates for datasets from general low-dimensional manifolds beyond the linear setting.

\section*{Funding}
T.S. acknowledges support from DOE-SC0023490, NIH
R01HL16351, NSF DMS-2027248, and NSF DMS-2208356. R.V. acknowledges support from
NSF DMS-1954233, NSF DMS-2027299, U.S. Army 76649-CS, and NSF+Simons Research
Collaborations on the Mathematical and Scientific Foundations of Deep Learning. Y.Z. is partially supported by NSF-Simons Research Collaborations on the Mathematical and Scientific Foundations of Deep Learning and an AMS-Simons Travel Grant.

\section*{Data Availability Statement}
No new data were generated or analysed in support of this review.





\bibliographystyle{plain}
\bibliography{ref}

\begin{thebibliography}{10}

\bibitem{abowd2019census}
John Abowd, Robert Ashmead, Garfinkel Simson, Daniel Kifer, Philip Leclerc,
  Ashwin Machanavajjhala, and William Sexton.
\newblock Census topdown: Differentially private data, incremental schemas, and
  consistency with public knowledge.
\newblock {\em US Census Bureau}, 2019.

\bibitem{abowd20222020}
John~M Abowd, Robert Ashmead, Ryan Cumings-Menon, Simson Garfinkel, Micah
  Heineck, Christine Heiss, Robert Johns, Daniel Kifer, Philip Leclerc, Ashwin
  Machanavajjhala, et~al.
\newblock The 2020 census disclosure avoidance system topdown algorithm.
\newblock {\em Harvard Data Science Review}, (Special Issue 2), 2022.

\bibitem{achlioptas2001fast}
Dimitris Achlioptas and Frank McSherry.
\newblock Fast computation of low-rank matrix approximation.
\newblock In {\em Proceedings of the thirty-P third annual ACM symposium on
  Theory of computing}, pages 611--618. ACM, 2001.

\bibitem{amin2019differentially}
Kareem Amin, Travis Dick, Alex Kulesza, Andres Munoz, and Sergei Vassilvitskii.
\newblock Differentially private covariance estimation.
\newblock {\em Advances in Neural Information Processing Systems}, 32, 2019.

\bibitem{arora2018differentially}
Raman Arora, Jalaj Upadhyay, et~al.
\newblock Differentially private robust low-rank approximation.
\newblock {\em Advances in neural information processing systems}, 31, 2018.

\bibitem{balog2018differentially}
Matej Balog, Ilya Tolstikhin, and Bernhard Sch{\"o}lkopf.
\newblock Differentially private database release via kernel mean embeddings.
\newblock In {\em International Conference on Machine Learning}, pages
  414--422. PMLR, 2018.

\bibitem{bellovin2019privacy}
Steven~M Bellovin, Preetam~K Dutta, and Nathan Reitinger.
\newblock Privacy and synthetic datasets.
\newblock {\em Stan. Tech. L. Rev.}, 22:1, 2019.

\bibitem{bhatia2013matrix}
Rajendra Bhatia.
\newblock {\em Matrix analysis}, volume 169.
\newblock Springer Science \& Business Media, 2013.

\bibitem{blum2013learning}
Avrim Blum, Katrina Ligett, and Aaron Roth.
\newblock A learning theory approach to noninteractive database privacy.
\newblock {\em Journal of the ACM (JACM)}, 60(2):1--25, 2013.

\bibitem{boedihardjo2022private2}
March Boedihardjo, Thomas Strohmer, and Roman Vershynin.
\newblock Private sampling: a noiseless approach for generating differentially
  private synthetic data.
\newblock {\em SIAM Journal on Mathematics of Data Science}, 4(3):1082--1115,
  2022.

\bibitem{boedihardjo2022covariance}
March Boedihardjo, Thomas Strohmer, and Roman Vershynin.
\newblock Covariance’s loss is privacy’s gain: Computationally efficient,
  private and accurate synthetic data.
\newblock {\em Foundations of Computational Mathematics}, 24(1):179--226, 2024.

\bibitem{boedihardjo2022private}
March Boedihardjo, Thomas Strohmer, and Roman Vershynin.
\newblock Private measures, random walks, and synthetic data.
\newblock {\em Probability theory and related fields}, pages 1--43, 2024.

\bibitem{bubeck2021universal}
S{\'e}bastien Bubeck and Mark Sellke.
\newblock A universal law of robustness via isoperimetry.
\newblock {\em Advances in Neural Information Processing Systems},
  34:28811--28822, 2021.

\bibitem{chaudhuri2011differentially}
Kamalika Chaudhuri, Claire Monteleoni, and Anand~D Sarwate.
\newblock Differentially private empirical risk minimization.
\newblock {\em Journal of Machine Learning Research}, 12(3), 2011.

\bibitem{chaudhuri2013near}
Kamalika Chaudhuri, Anand~D Sarwate, and Kaushik Sinha.
\newblock A near-optimal algorithm for differentially-private principal
  components.
\newblock {\em Journal of Machine Learning Research}, 14, 2013.

\bibitem{dai2022tail}
Guozheng Dai, Zhonggen Su, and Hanchao Wang.
\newblock Tail bounds on the spectral norm of sub-exponential random matrices.
\newblock {\em Random Matrices: Theory and Applications}, 13(01):2350013, 2024.

\bibitem{davis1970rotation}
Chandler Davis and William~Morton Kahan.
\newblock The rotation of eigenvectors by a perturbation. iii.
\newblock {\em SIAM Journal on Numerical Analysis}, 7(1):1--46, 1970.

\bibitem{dong2022differentially}
Wei Dong, Yuting Liang, and Ke~Yi.
\newblock Differentially private covariance revisited.
\newblock {\em Advances in Neural Information Processing Systems}, 35:850--861,
  2022.

\bibitem{donhauser2023sample}
Konstantin Donhauser, Johan Lokna, Amartya Sanyal, March Boedihardjo, Robert
  H{\"o}nig, and Fanny Yang.
\newblock Certified private data release for sparse lipschitz functions.
\newblock In {\em International Conference on Artificial Intelligence and
  Statistics}, pages 1396--1404. PMLR, 2024.

\bibitem{dwork2006our}
Cynthia Dwork, Krishnaram Kenthapadi, Frank McSherry, Ilya Mironov, and Moni
  Naor.
\newblock Our data, ourselves: Privacy via distributed noise generation.
\newblock In {\em Advances in Cryptology-EUROCRYPT 2006: 24th Annual
  International Conference on the Theory and Applications of Cryptographic
  Techniques, St. Petersburg, Russia, May 28-June 1, 2006. Proceedings 25},
  pages 486--503. Springer, 2006.

\bibitem{dwork2019differential}
Cynthia Dwork, Nitin Kohli, and Deirdre Mulligan.
\newblock Differential privacy in practice: Expose your epsilons!
\newblock {\em Journal of Privacy and Confidentiality}, 9(2), 2019.

\bibitem{dwork2009complexity}
Cynthia Dwork, Moni Naor, Omer Reingold, Guy~N Rothblum, and Salil Vadhan.
\newblock On the complexity of differentially private data release: efficient
  algorithms and hardness results.
\newblock In {\em Proceedings of the forty-first annual ACM symposium on Theory
  of computing}, pages 381--390, 2009.

\bibitem{dwork2015efficient}
Cynthia Dwork, Aleksandar Nikolov, and Kunal Talwar.
\newblock Efficient algorithms for privately releasing marginals via convex
  relaxations.
\newblock {\em Discrete \& Computational Geometry}, 53:650--673, 2015.

\bibitem{dwork2014algorithmic}
Cynthia Dwork and Aaron Roth.
\newblock The algorithmic foundations of differential privacy.
\newblock {\em Foundations and Trends{\textregistered} in Theoretical Computer
  Science}, 9(3--4):211--407, 2014.

\bibitem{dwork2014analyze}
Cynthia Dwork, Kunal Talwar, Abhradeep Thakurta, and Li~Zhang.
\newblock Analyze {G}auss: optimal bounds for privacy-preserving principal
  component analysis.
\newblock In {\em Proceedings of the forty-sixth annual ACM symposium on Theory
  of computing}, pages 11--20, 2014.

\bibitem{feige2005spectral}
Uriel Feige and Eran Ofek.
\newblock Spectral techniques applied to sparse random graphs.
\newblock {\em Random Structures \& Algorithms}, 27(2):251--275, 2005.

\bibitem{harder2021dp}
Frederik Harder, Kamil Adamczewski, and Mijung Park.
\newblock Dp-merf: Differentially private mean embeddings with random features
  for practical privacy-preserving data generation.
\newblock In {\em International conference on artificial intelligence and
  statistics}, pages 1819--1827. PMLR, 2021.

\bibitem{hardt2012simple}
Moritz Hardt, Katrina Ligett, and Frank McSherry.
\newblock A simple and practical algorithm for differentially private data
  release.
\newblock {\em Advances in neural information processing systems}, 25, 2012.

\bibitem{hardt2014noisy}
Moritz Hardt and Eric Price.
\newblock The noisy power method: A meta algorithm with applications.
\newblock {\em Advances in neural information processing systems}, 27, 2014.

\bibitem{hardt2013beyond}
Moritz Hardt and Aaron Roth.
\newblock Beyond worst-case analysis in private singular vector computation.
\newblock In {\em Proceedings of the forty-fifth annual ACM symposium on Theory
  of computing}, pages 331--340, 2013.

\bibitem{he2023algorithmically}
Yiyun He, Roman Vershynin, and Yizhe Zhu.
\newblock Algorithmically effective differentially private synthetic data.
\newblock In Gergely Neu and Lorenzo Rosasco, editors, {\em Proceedings of
  Thirty Sixth Conference on Learning Theory}, volume 195 of {\em Proceedings
  of Machine Learning Research}, pages 3941--3968. PMLR, 12--15 Jul 2023.

\bibitem{imtiaz2016symmetric}
Hafiz Imtiaz and Anand~D Sarwate.
\newblock Symmetric matrix perturbation for differentially-private principal
  component analysis.
\newblock In {\em 2016 IEEE International Conference on Acoustics, Speech and
  Signal Processing (ICASSP)}, pages 2339--2343. IEEE, 2016.

\bibitem{inusah2006discrete}
Seidu Inusah and Tomasz~J Kozubowski.
\newblock A discrete analogue of the laplace distribution.
\newblock {\em Journal of statistical planning and inference},
  136(3):1090--1102, 2006.

\bibitem{jain2016streaming}
Prateek Jain, Chi Jin, Sham~M Kakade, Praneeth Netrapalli, and Aaron Sidford.
\newblock Streaming {PCA}: Matching matrix {B}ernstein and near-optimal finite
  sample guarantees for {O}ja’s algorithm.
\newblock In {\em Conference on learning theory}, pages 1147--1164. PMLR, 2016.

\bibitem{jiang2016wishart}
Wuxuan Jiang, Cong Xie, and Zhihua Zhang.
\newblock Wishart mechanism for differentially private principal components
  analysis.
\newblock {\em Proceedings of the AAAI Conference on Artificial Intelligence},
  30(1), 2016.

\bibitem{jiang2013differential}
Xiaoqian Jiang, Zhanglong Ji, Shuang Wang, Noman Mohammed, Samuel Cheng, and
  Lucila Ohno-Machado.
\newblock Differential-private data publishing through component analysis.
\newblock {\em Transactions on data privacy}, 6(1):19, 2013.

\bibitem{kamath2019privately}
Gautam Kamath, Jerry Li, Vikrant Singhal, and Jonathan Ullman.
\newblock Privately learning high-dimensional distributions.
\newblock In {\em Conference on Learning Theory}, pages 1853--1902. PMLR, 2019.

\bibitem{kapralov2013differentially}
Michael Kapralov and Kunal Talwar.
\newblock On differentially private low rank approximation.
\newblock In {\em Proceedings of the twenty-fourth annual ACM-SIAM symposium on
  Discrete algorithms}, pages 1395--1414. SIAM, 2013.

\bibitem{kovalev2022lipschitz}
Leonid~V Kovalev.
\newblock {L}ipschitz clustering in metric spaces.
\newblock {\em The Journal of Geometric Analysis}, 32(7):188, 2022.

\bibitem{kreacic2023differentially}
Eleonora Kreacic, Navid Nouri, Vamsi~K Potluru, Tucker Balch, and Manuela
  Veloso.
\newblock Differentially private synthetic data using kd-trees.
\newblock In {\em The 39th Conference on Uncertainty in Artificial
  Intelligence}, 2023.

\bibitem{li2019tutorial}
Xiaocan Li, Shuo Wang, and Yinghao Cai.
\newblock Tutorial: Complexity analysis of singular value decomposition and its
  variants.
\newblock {\em arXiv preprint arXiv:1906.12085}, 2019.

\bibitem{liu2022dp}
Xiyang Liu, Weihao Kong, Prateek Jain, and Sewoong Oh.
\newblock {DP}-{PCA}: Statistically optimal and differentially private pca.
\newblock {\em Advances in neural information processing systems},
  35:29929--29943, 2022.

\bibitem{liu2021robust}
Xiyang Liu, Weihao Kong, Sham Kakade, and Sewoong Oh.
\newblock Robust and differentially private mean estimation.
\newblock {\em Advances in neural information processing systems},
  34:3887--3901, 2021.

\bibitem{liu2022differential}
Xiyang Liu, Weihao Kong, and Sewoong Oh.
\newblock Differential privacy and robust statistics in high dimensions.
\newblock In {\em Conference on Learning Theory}, pages 1167--1246. PMLR, 2022.

\bibitem{mangoubi2022re}
Oren Mangoubi and Nisheeth Vishnoi.
\newblock Re-analyze {G}auss: Bounds for private matrix approximation via
  {D}yson {B}rownian motion.
\newblock {\em Advances in Neural Information Processing Systems},
  35:38585--38599, 2022.

\bibitem{meunier2022dynamical}
Laurent Meunier, Blaise~J Delattre, Alexandre Araujo, and Alexandre Allauzen.
\newblock A dynamical system perspective for {L}ipschitz neural networks.
\newblock In {\em International Conference on Machine Learning}, pages
  15484--15500. PMLR, 2022.

\bibitem{oja1982simplified}
Erkki Oja.
\newblock Simplified neuron model as a principal component analyzer.
\newblock {\em Journal of mathematical biology}, 15:267--273, 1982.

\bibitem{reiss2020nonasymptotic}
Markus Reiss and Martin Wahl.
\newblock Nonasymptotic upper bounds for the reconstruction error of {PCA}.
\newblock {\em The Annals of Statistics}, 48(2):1098--1123, 2020.

\bibitem{singhal2021privately}
Vikrant Singhal and Thomas Steinke.
\newblock Privately learning subspaces.
\newblock {\em Advances in Neural Information Processing Systems},
  34:1312--1324, 2021.

\bibitem{thaler2012faster}
Justin Thaler, Jonathan Ullman, and Salil Vadhan.
\newblock Faster algorithms for privately releasing marginals.
\newblock In {\em Automata, Languages, and Programming: 39th International
  Colloquium, ICALP 2012, Warwick, UK, July 9-13, 2012, Proceedings, Part I
  39}, pages 810--821. Springer, 2012.

\bibitem{tsfadia2024differentially}
Eliad Tsfadia.
\newblock On differentially private subspace estimation in a distribution-free
  setting.
\newblock {\em arXiv preprint arXiv:2402.06465}, 2024.

\bibitem{ullman2011pcps}
Jonathan Ullman and Salil Vadhan.
\newblock {PCP}s and the hardness of generating private synthetic data.
\newblock In {\em Theory of Cryptography: 8th Theory of Cryptography
  Conference, TCC 2011, Providence, RI, USA, March 28-30, 2011. Proceedings 8},
  pages 400--416. Springer, 2011.

\bibitem{vietriprivate}
Giuseppe Vietri, Cedric Archambeau, Sergul Aydore, William Brown, Michael
  Kearns, Aaron Roth, Ankit Siva, Shuai Tang, and Steven Wu.
\newblock Private synthetic data for multitask learning and marginal queries.
\newblock In {\em Advances in Neural Information Processing Systems}, 2022.

\bibitem{villani2009optimal}
C{\'e}dric Villani.
\newblock {\em Optimal transport: old and new}, volume 338.
\newblock Springer, 2009.

\bibitem{von2004distance}
Ulrike von Luxburg and Olivier Bousquet.
\newblock Distance-based classification with {L}ipschitz functions.
\newblock {\em J. Mach. Learn. Res.}, 5(Jun):669--695, 2004.

\bibitem{wang2016differentially}
Ziteng Wang, Chi Jin, Kai Fan, Jiaqi Zhang, Junliang Huang, Yiqiao Zhong, and
  Liwei Wang.
\newblock Differentially private data releasing for smooth queries.
\newblock {\em The Journal of Machine Learning Research}, 17(1):1779--1820,
  2016.

\bibitem{wasserman2010statistical}
Larry Wasserman and Shuheng Zhou.
\newblock A statistical framework for differential privacy.
\newblock {\em Journal of the American Statistical Association},
  105(489):375--389, 2010.

\bibitem{yang2023differentially}
Yilin Yang, Kamil Adamczewski, Danica~J Sutherland, Xiaoxiao Li, and Mijung
  Park.
\newblock Differentially private neural tangent kernels for privacy-preserving
  data generation.
\newblock {\em arXiv preprint arXiv:2303.01687}, 2023.

\bibitem{yu2015useful}
Yi~Yu, Tengyao Wang, and Richard~J Samworth.
\newblock A useful variant of the {D}avis--{K}ahan theorem for statisticians.
\newblock {\em Biometrika}, 102(2):315--323, 2015.

\bibitem{zhou2009differential}
Shuheng Zhou, Katrina Ligett, and Larry Wasserman.
\newblock Differential privacy with compression.
\newblock In {\em 2009 IEEE International Symposium on Information Theory},
  pages 2718--2722. IEEE, 2009.

\end{thebibliography}


\appendix

\end{document}